\documentclass[11pt]{article}
\usepackage{latexsym,amssymb,amsmath} 
\usepackage{path}
\usepackage{verbatim}
\usepackage{amsfonts}
\usepackage{amsfonts}
\usepackage{array}

\usepackage{mathrsfs}
\usepackage{amsfonts}
\usepackage{amsmath}
\usepackage{amssymb}
\usepackage{pifont}
\usepackage{amsthm}
\usepackage{bm}
\usepackage{appendix}
\usepackage{multirow}
\usepackage{pifont}
\usepackage{graphicx} 
\usepackage{subfigure}
\usepackage{color}
\usepackage[table]{xcolor}
\usepackage{thmtools}
\usepackage{thm-restate}
\usepackage{enumitem}

\usepackage{graphicx}
\usepackage{epsfig}
\usepackage{subfigure}
\usepackage{epstopdf}

\usepackage{tcolorbox}

\usepackage[pagebackref=true,colorlinks=true,citecolor=blue,bookmarks=false]{hyperref}

\definecolor{gray}{rgb}{0.85,0.85,0.85}
\definecolor{yama}{rgb}{0.98, 0.87, 0.68}
\definecolor{lightskyblue}{rgb}{0.53, 0.81, 0.98}

\graphicspath{{figures/}}

\usepackage[lined,boxed,ruled, commentsnumbered]{algorithm2e}

\newtheorem{lemma}{Lemma}
\newtheorem{theorem}{Theorem}
\newtheorem{corollary}{Corollary}

\newtheorem{definition}{Definition}
\newtheorem{remark}{Remark}
\newtheorem{assumption}{Assumption}

\newcommand{\supp}{\text{supp}}

\newcommand{\argmin}{\mathop{\arg\min}}

\newcommand{\sign}{\text{sign}}

\newcommand{\st}{\text{subject to }}


\newcommand{\junk}[1]{{}}

\newlength{\fwtwo} 

\DeclareGraphicsExtensions{.jpg}

\title{Stability and Risk Bounds of Iterative Hard Thresholding}
\author{
  \textbf{Xiao-Tong Yuan} \ and \ \textbf{Ping Li} \\\\
   Cognitive Computing Lab \\
  Baidu Research \\
   No. 10 Xibeiwang East Road, Beijing 100193, China \\
   10900 NE 8th St. Bellevue, Washington 98004, USA\\
   E-mail: \texttt{\{xtyuan1980, pingli98\}@gmail.com}
  }
\date{}

\begin{document}

\maketitle

\begin{abstract}
\noindent In this paper,\footnote{A conference version of this article appeared in the Proceedings of the 24th International Conference on Artificial Intelligence and Statistics (AISTATS) 2021.} we analyze the generalization performance of the Iterative Hard Thresholding (IHT) algorithm widely used for sparse recovery problems. The parameter estimation and sparsity recovery consistency of IHT has long been known in compressed sensing. From the perspective of statistical learning, another fundamental question is how well the IHT estimation would predict on unseen data. This paper makes progress towards answering this open question by introducing a novel sparse generalization theory for IHT under the notion of algorithmic stability. Our theory reveals that: 1) under natural conditions on the empirical risk function over $n$ samples of dimension $p$, IHT with sparsity level $k$ enjoys an $\mathcal{\tilde O}(n^{-1/2}\sqrt{k\log(n)\log(p)})$ rate of convergence in sparse excess risk; 2) a tighter $\mathcal{\tilde O}(n^{-1/2}\sqrt{\log(n)})$ bound can be established by imposing an additional iteration stability condition on a hypothetical IHT procedure invoked to the population risk; and 3) a fast rate of order $\mathcal{\tilde O}\left(n^{-1}k(\log^3(n)+\log(p))\right)$ can be derived for strongly convex risk function under proper strong-signal conditions. The results have been substantialized to sparse linear regression and sparse logistic regression models to demonstrate the applicability of our theory. Preliminary numerical evidence is provided to confirm our theoretical predictions.

\end{abstract}


\newpage
\section{Introduction}

We are interested in developing \emph{sparse learning theory} for the  following problem of high-dimensional stochastic risk minimization under hard sparsity constraint:
\[
\min_{w \in \mathcal{W}} F(w):= \mathbb{E}_{\xi \sim D} [\ell(w;\xi)] \quad \st \|w\|_0\le k,
\]
where $w\in \mathcal{W}\subseteq \mathbb{R}^p$ is the model parameter vector, $\ell(w;\xi)$ is a non-negative convex function that measures the loss of $w$ at a data instance $\xi \in \mathcal{X}$, $D$ represents a random distribution over $\mathcal{X}$. The cardinality constraint $\|w\|_0\le k$ is imposed for enhancing the interpretability and learnability of model in cases where there are no clear favourite explanatory variables or the model is overparameterized. Such a sparse risk minimization problem, which is also known as prediction with best subset selection, has been widely studied in statistical learning~\cite{greenshtein2006best,greenshtein2004persistence} and econometrics~\cite{chen2018best,chen2020binary,jiang2010risk}.

In realistic problems, the mathematical formulation of $D$ is typically unknown and thus it is hopeless to directly optimize such a stochastic formulation. Alternatively, given a set of i.i.d. training samples $S=\{\xi_i\}_{i=1}^n \in \mathcal{X}^n$ drawn from $D$, the following sparsity-constrained empirical risk minimization problem is often considered for learning sparse models in high-dimensional settings~\cite{bach2012optimization,donoho2006compressed,hastie2015statistical}:
\begin{equation}\label{prob:general}
\min_{w \in \mathcal{W}} F_S(w):= \frac{1}{n}\sum\limits_{i=1}^n \ell(w;\xi_i) \quad \st \|w\|_0\le k.
\end{equation}
Here the cardinality constraint is crucial for accurate estimation especially when $p\gg n$ which is often the case in big data era. The above sparse M-estimation model will be referred to as $\ell_0$-ERM in this work.

Due to the presence of cardinality constraint, the $\ell_0$-ERM estimator is simultaneously non-convex and NP-hard even when the loss function is quadratic~\cite{natarajan1995sparse}, which makes it computationally intractable to solve the problem exactly in general cases. Therefore, one must seek approximate solutions instead of carrying out combinatorial search over all possible models. Among others, Iterative Hard Thresholding (IHT)~\cite{blumensath2009iterative} is a family of first-order greedy selection methods popularly used for approximately solving $\ell_0$-ERM with strong theoretical guarantees and outstanding practical efficiency~\cite{jain2014iterative,jin2016training,yuan2018gradient,zhou2018efficient}. The common theme of IHT-style algorithms is to iterate between gradient descent and hard thresholding to decrease the objective value while maintaining sparsity of solution. In the considered problem setting, a plain IHT algorithm generates a sequence $\{ w_{S,k}^{(t)}\}_{t\ge 1}$ according to the following recursion form with learning rate $\eta>0$:
\begin{equation}\label{prob:iht}
 w_{S,k}^{(t)}:= \mathrm{H}_k \left(w_{S,k}^{(t-1)} - \eta\nabla F_S(w_{S,k}^{(t-1)})\right),
\end{equation}
where $\mathrm{H}_k(\cdot)$ is the truncation operator that preserves the top $k$ (in magnitude) entries of input and sets the remaining to be zero, with ties broken arbitrarily. The procedure is typically initialized with all-zero vector, i.e., $w_{S,k}^{(0)}=0$. The IHT-style algorithms have been known to converge linearly towards certain nominal sparse model with optimal estimation accuracy~\cite{bahmani2013greedy,yuan20nearly,yuan2014gradient} under proper regularity conditions. In practice, IHT-style algorithms have found their applications in deep neural networks pruning~\cite{jin2016training}, sparse signal demixing from noisy observations~\cite{soltani2017fast}, and few-shot learning for image classification~\cite{tian2020meta}, to name a few.

\subsection{Problem and motivation}

In this paper, we are interested in the following question about the generalization performance~of~IHT:

\begin{tcolorbox}[width=1.0\linewidth, boxsep=0pt,left=1pt,top=1.5pt,bottom=1.5pt,boxrule=1pt]
\emph{How well the corresponding population risk $F(w^{(t)}_{S,k})= \mathbb{E}_{\xi \sim D} [\ell(w^{(t)}_{S,k};\xi)]$ can approximate the optimal sparse population risk $F(\bar w)= \min_{\|w\|_0\le \bar k} F(w)$ for $\bar k \le k$?}
\end{tcolorbox}

The answer to this question is important for understanding the generalization ability of IHT yet has remaind elusive. In what follows, the value $F(w^{(t)}_{S,k}) - F(\bar w)$ is referred to as the \emph{$\bar k$-sparse excess risk} of IHT. The primary goal of this study is to derive a suitable law of large numbers, i.e., a sample size vanishing rate $\gamma_n$ such that the sparse excess risk bound $F(w^{(t)}_{S,k}) - F(\bar w) \le \gamma_n$ holds with high probability. The standard sparse learning paradigm usually assumes that there exits a true sparse parameter vector for data generalization, and the recovery or prediction behaviour of estimators is studied as the number of observations increases with the true model kept fixed. In such a problem regime with well-specified model sparsity, the sparse excess risk bounds of IHT for smooth loss functions can be readily implied by the classical sparse parameter estimation bounds~\cite{jain2014iterative,yuan2018gradient} (see Section~\ref{apdsect:fast_rates_well_specified} for more detailed discussions on this line of results).

For the present study, we adopt a substantially different statistical paradigm. We mainly focus on the performance of IHT for risk minimization with predictor selection rather than the estimation of a fixed unknown sparsity model, which may or may not exist in real-life learning tasks. That is, we do not require the $\bar k$-sparse minimizer $\bar w=\argmin_{\|w\|_0\le \bar k} F(w)$ to be a true model of data generalization, and we do not attempt to derive the risk bound of IHT via explicitly estimating $\bar w$. This sparse learning paradigm with potentially misspecified model sparsity has received wide attention in the prediction error analysis of $\ell_1$-penalized estimations (Lasso) and $\ell_0$-ERM~\cite{chen2018best,foster2019statistical,van2008high}, the binary choice prediction of econometric time series~\cite{jiang2010risk}, and meta-learning under network capacity constraint~\cite{tian2020meta}. Particularly for $\ell_0$-ERM, up to logarithmic factors, a number of $\mathcal{\tilde O}(\sqrt{k/n})$ uniform excess risk bounds have recently been derived for binary loss functions~\cite{chen2018best,chen2020binary}, and tighter bounds of order $\mathcal{\tilde O}(k/n)$ were established for bounded liner prediction classes~\cite{foster2019statistical}. These existing sparse risk bounds for $\ell_0$-ERM, however, are obtained under an unrealistic condition that its global minimizer is exactly available. It is not yet clear if these known results of $\ell_0$-ERM can be extended to computationally tractable sparsity recovery algorithms such as IHT. {In the meanwhile, for misspecified sparsity models with smooth losses,  a naive application of the classic parameter estimation error bounds (see, e.g.,~\cite{yuan2018gradient}) would yield excess risk bounds of order $ \mathcal{\tilde O}\left(k\|\nabla F(\bar w)\|^2_\infty + \frac{k}{n}\right)$ which is substantially inferior to those of $\ell_0$-ERM as $\nabla F(\bar w)\neq 0$ typically holds. It still remains an open question if sharper risk bounds can possibly be derived for IHT to match those known for $\ell_0$-ERM~\cite{chen2018best,chen2020binary,foster2019statistical} in misspecified cases.}

Alternatively, a useful and popular proxy for analyzing the generalization performance is the \emph{stability} of learning algorithms to changes in the training dataset~\cite{bousquet2002stability}. By hinging the optimality of ERM, stability has been extensively demonstrated to beget strong generalization bounds for ERM solutions with convex loss functions~\cite{mukherjee2006learning,shalev2010learnability} and for iterative learning algorithms (such as SGD) as well~\cite{charles2018stability,hardt2016train,kuzborskij2018data}. Specially, the state-of-the-art generalization results for strongly convex ERM are offered by approaches based on the notion of uniform stability~\cite{bousquet2020sharper,feldman2018generalization,feldman2019high}. Inspired by the remarkable success of stability theory, we aim at deriving sparse excess risk bounds for IHT via the uniform stability arguments, which to our knowledge has not been systematically treated elsewhere in literature.

Yet, the traditional uniform stability arguments of regularized convex ERM do not naturally extend to IHT. The crux here is that the stability of IHT relies heavily on the stability of its recovered supporting set, $\supp(w^{(t)}_{S,k})$, which could be highly non-trivial to guarantee even that the empirical risk function is strongly convex. In contrast, the convectional dense ERM is supported over the entire range of feature dimension and thus its supporting set is by nature unique~and~stable.

\subsection{Overview of our work and main results}

We offer two solutions to address the above mentioned stability issue about the sparsity pattern of IHT. The idea of the first solution is intuitive in principle:  \emph{If the empirical risk $F_S$ has restricted strong convexity and smoothness, then based on the uniform stability of ERM restricted over any feature index set of cardinality $k$ we can establish a high probability generalization bound for IHT via applying union probability arguments to all the possible $k$-sparse supporting sets.} A main technical obstacle we need to overcome for this strategy is that in many statistical learning problems the restricted strong convexity of the empirical risk usually holds with high probability over data sample rather than uniformly. As a new element of our analysis for dealing with such a small failure probability of strong convexity, we propose to analyze IHT when applied to a regularized variant of $\ell_0$-ERM with a penalty term $\mathcal{O}(n^{-1/2}\|w\|^2)$ added to guarantee restricted uniform stability, and consequently show that the stability-induced risk bound of the regularized IHT estimator can be inherited by the original IHT with high chance. The corresponding main result in Theorem~\ref{thrm:uniform_stability_iht} shows that the sparse excess risk of IHT can be upper bounded by $\mathcal{\tilde O}\left(n^{-1/2}\sqrt{k\log(n)\log(ep/k)}\right)$ with high probability over data sample. This bound is comparable to that established in~\cite[Theorem 1]{chen2018best} for the problem of binary prediction with best subset selection. In contrast to that result developed for the exact solution of $\ell_0$-ERM, our result is applicable to the IHT algorithm for approximately solving $\ell_0$-ERM.

Our second attempt is to directly analyze the stability of IHT with respect to its recovered supporting set. The key ingredient here is to show that by imposing some additional stability conditions on the population risk $F$, the support recovery of IHT would be stable with high probability. More precisely, we will show that if $F$ is \emph{stable with respect to IHT up to the desired rounds of iteration} (see Definition~\ref{def:iht_stability} for a formal definition), then it holds with high probability that the empirical IHT estimator $w^{(t)}_{S,k}$ is also stable in support recovery given that the sample size is sufficiently large. As the main result in this regime, we establish in Theorem~\ref{thrm:uniform_stability_strong_iht} an $\mathcal{\tilde O}(n^{-1/2}\sqrt{\log(n)})$ high probability excess risk bound for IHT which matches a near-optimal (up to logarithmic factors) bound for regularized ERM without sparsity constraint~\cite{feldman2019high}.

The $\mathcal{\tilde O}(n^{-1/2})$ rates of convergence established in Theorem~\ref{thrm:uniform_stability_iht} and Theorem~~\ref{thrm:uniform_stability_strong_iht} are usually referred to as slow rates in statistical learning theory. For strongly convex risk minimization problems, we further derive in Theorem~\ref{thrm:generalizaion_fast_rate} a fast rate of order  $\mathcal{\tilde O}\left(n^{-1}k(\log^3(n)+\log(p))\right)$ for IHT under additional strong-signal conditions. The key observation is that when the signal strength of the target sparse optimal solution is sufficiently strong, then the support of the target solution can be recovered as a subset of that of the IHT estimation. Consequently, the desired fast rate of convergence can be derived via invoking the theory of local Rademacher complexities~\cite{bartlett2005local} over the supporting set of IHT. This result matches the $\mathcal{\tilde O}(n^{-1}k\log(p))$ rate established for $\ell_0$-ERM~\cite[Example 2]{foster2019orthogonal}, showing that IHT generalizes as efficiently as the exact $\ell_0$-ERM solver for strongly convex problems. In comparison to a similar fast rate of Lasso for well-specified sparsity models~\cite{van2008high}, our result in Theorem~\ref{thrm:generalizaion_fast_rate} is applicable to misspecified sparsity models as well and thus is more general. Further, specially for well-specified sparse learning models such as sparse generalized linear models, we show through Theorem~\ref{thrm:generalization_barw_iht} that an $\mathcal{\tilde O}(n^{-1}k\log(p))$ fast rate of convergence can be more directly derived based on the existing parameter estimation error bounds of IHT under~mild~conditions~\cite{jain2014iterative,yuan2018gradient}.

\begin{table}[b!]
\centering
\begin{tabular}{|c|c|c|c|c|}
\hline
 Result  &  Risk Bound &  Model Sparsity & Key Conditions \\
\hline
Theorem~\ref{thrm:uniform_stability_iht} & $\mathcal{\tilde O} \left(\sqrt{\frac{k\log(n)\log(ep/k)}{n}}\right)$  & Misspecified & RSC/RLS, Lipschitz-loss \\
\hline
Theorem~\ref{thrm:uniform_stability_strong_iht} & $\mathcal{\tilde O} \left(\sqrt{\frac{\log(n)}{n}}\right)$ & Misspecified & RSC/RLS,  IHT-Stability \\
\hline
Theorem~\ref{thrm:generalizaion_fast_rate} & $\mathcal{\tilde O} \left(\frac{k(\log^3(n)+\log(p))}{n}\right)$ & Misspecified & RSC/RLS, Strong-signal \\
\hline
Theorem~\ref{thrm:generalization_barw_iht} & $\mathcal{\tilde O} \left(\frac{k\log(p)}{n}\right)$ & Well-specified & RSC/RLS \\
\hline
\end{tabular}
\caption{Overview of our main results on the sparse excess risk bounds of IHT. The big $\mathcal{\tilde O}$ notation hides the logarithmic factors on tail bound. RSC and RLS respectively stand for Restricted Strongly Convexity and Restricted Lipschitz Smoothness (see Definition~\ref{def:strong_smooth}). The concept of IHT-Stability is defined in Definition~\ref{def:iht_stability}. \label{tab:results}}
\end{table}

In a nutshell, this paper establishes a set of algorithm stability induced sparse excess risk bounds for IHT without imposing any distribution-specific assumptions on the data generation model. As a side contribution, we have also derived a fast rate of convergence for IHT when the data is assumed to be generated by a well-specified sparse model. Our main results and the related key model assumptions and technical conditions are highlighted in Table~\ref{tab:results}. The connections and differences of our results to the prior existing risk bounds for $\ell_0$-ERM and Lasso-type estimators are elaborated in detail in Section~\ref{sect:comparison}. To demonstrate the applicability of our theory, we have substantialized these risk bounds to the widely used sparse linear regression and logistic regression models, along with numerical evidences to support~the~theoretical~predictions.

\subsection{Paper organization}

The paper proceeds with the material organized as follows: In Section~\ref{sect:related_work}, we briefly review the related literature. In Section~\ref{sect:stability_risk_bounds} and  Section~\ref{sect:stability_risk_bounds_fast} we respectively present a set of slow and fast sparse excess risk bounds of IHT via uniform stability arguments. A comparison of our results to some prior relevant results is provided in Section~\ref{sect:comparison}. A preliminary numerical study for theory verification is provided in Section~\ref{sect:simulation_study}. The concluding remarks are made in Section~\ref{sect:conclusion}. All the technical proofs are relegated to the appendix sections.

\section{Related Work}
\label{sect:related_work}

The problem regime considered in this paper lies at the intersection of high-dimensional sparse M-estimation and statistical learning theory, both of which have long been studied with a vast body of beautiful and deep theoretical results established in literature. Next we will incompletely connect our research to several closely relevant lines of study in this context. We refer the interested readers to~\cite{cesa2006prediction,hastie2015statistical,wainwright2019high} and the references therein for a more comprehensive coverage of the related topics.

\textbf{Consistency and generalization of M-estimation with sparsity.} Statistical consistency of learning with sparsity models is now well understood for some popular sparse M-estimators including $\ell_0$-ERM~\eqref{prob:general}~\cite{foucart2017mathematical,rigollet201518,yuan2016exact}, Lasso~\cite{tibshirani1996regression,bellec2018slope,loh2012high,meinshausen2009lasso,wainwright2009sharp} and folded concave penalization~\cite{fan2001variable,fan2014strong,zhang2010nearly,zhang2012general}. The generalization ability of sparsity-inducing learning models is relatively less understood but has gained recent significant attention. The excess risk of Lasso for generalized linear models was investigated in~\cite{van2008high}. Later with almost no assumptions imposed on the design matrix, the least squares Lasso estimator was still shown to be consistent in out-of-sample predictive risk~\cite{chatterjee2013assumptionless}. For a class of $\ell_1$-penalized high dimensional M-estimators with non-convex loss functions, uniform convergence bounds with polynomial dependence on the sparsity level of certain nominal model were established in~\cite{mei2018landscape}. The misclassification excess risk of sparsity-penalized binary logistic regression has been investigated in~\cite{abramovich2018high} with near-optimal high probability bounds established. For linear prediction models, a data dependent generalization error bound was derived for a class of risk minimization algorithms with structured sparsity constraints~\cite{maurer2012structured}. Particularly concerning the generalization of $\ell_0$-ERM, a set of uniform excess risk bounds were derived in~\cite{chen2018best,chen2020binary} for binary loss functions under proper regularity conditions. More recently, based on the arguments of localized Rademacher complexity~\cite{bartlett2005local}, tighter risk bounds for $\ell_0$-ERM have been established over bounded liner prediction classes~\cite{foster2019statistical}. The existing uniform convergence implied excess risk bounds for $\ell_0$-ERM, however, rely largely upon its optimal solution which is NP-hard to be estimated exactly in high-dimensional setting. It is not yet clear if these results can be extended to approximate sparsity recovery algorithms such as IHT considered in this work.

\textbf{Statistical guarantees on IHT-style algorithms.} The IHT-style algorithms have been popularly applied and studied in compressed sensing and sparse learning~\cite{blumensath2009iterative,foucart2011hard,garg2009gradient}. Recent works have demonstrated that by imposing certain assumptions such as restricted strong convexity/smothness and restricted isometry property (RIP) over the risk function, IHT and its variants converge linearly towards certain nominal sparse model with near-optimal estimation accuracy~\cite{bahmani2013greedy,yuan2014gradient}. It was later shown in~\cite{jain2014iterative,shen2017tight} that with proper relaxation of sparsity level, high-dimensional estimation consistency can be established for IHT without assuming RIP conditions. The sparsity recovery performance of IHT-style methods was investigated in~\cite{shen2017iteration,yuan2016exact} to understand when the algorithm can exactly recover the support of a sparse signal from its compressed measurements. The excess risk analysis of IHT yet still remains an open challenge that we aim to attack in this work.

\textbf{Stability and generalization of ERM.} The idea of using stability of the algorithm with respect to changes in the training set for generalization error analysis dates back to the seventies~\cite{rogers1978finite,devroye1979distribution}. Since the seminal work of Bousquet and Elisseeff~\cite{bousquet2002stability}, stability has been extensively studied with a bunch of applications to establishing generalization bounds for strongly convex ERM estimators~\cite{zhang2003leave,mukherjee2006learning,shalev2010learnability}. Recently, it was shown that the solution obtained via (stochastic) gradient descent is expected to be stable and generalize well for smooth convex and non-convex loss functions~\cite{hardt2016train}. Later, a set of data-dependent generalization bounds for SGD were derived based on the stability of algorithm~\cite{kuzborskij2018data}. More broadly, generalization bounds for stable learning algorithms (e.g., GD, SGD and SVRG) that converge to global minima were established in~\cite{charles2018stability}. There is a recent renewed interest in the use of uniform stability for deriving high probability risk bounds of strongly convex ERM and optimization algorithms~\cite{bousquet2020sharper,feldman2018generalization,feldman2019high}. We highlight that our generalization analysis of IHT is a novel extension of the uniform stability theory in the direction of non-convex sparse learning under hard sparsity constraint.

\section{Sparse Excess Risk Bounds of IHT}
\label{sect:stability_risk_bounds}

In this section, we analyze the sparse excess risk bounds of IHT through the lens of algorithmic stability theory. We distinguish our analysis in two regimes. In the first regime, we establish an excess risk bound of IHT induced by the uniform stability of strongly convex ERM restricted over arbitrary feature set of cardinality $k$. In the second regime, we directly analyze the stability of IHT for support recovery which in turn leads to a stronger risk bound under  more stringent~conditions.

\subsection{Preliminaries}

We begin by introducing some definitions and basic assumptions which will be used  in the analysis to follow. The concept of uniform stability, as formally defined in below, is a powerful tool for analyzing generalization bounds of M-estimators and their learning algorithms as well~\cite{bousquet2002stability,feldman2019high,hardt2016train,shalev2010learnability}.

\newpage

\begin{definition}[Uniform Stability]\label{def:uniform_stability}
Let $A: \mathcal{X}^n \mapsto \mathcal{W}$ be a learning algorithm that maps a dataset $S \in \mathcal{X}^n$ to a model $A(S) \in \mathcal{W}$. $A$ is said to have uniform stability $\gamma$ with respect to a loss function $\ell: \mathcal{W} \times \mathcal{X}\mapsto \mathbb{R}$ if for any pair of datasets $S, S' \in \mathcal{X}^n$ that differ in a single element and every $x \in \mathcal{X}$, $\left|\ell(A(S); x) - \ell(A(S');x)\right| \le  \gamma$.
\end{definition}
For instance, conventional ERM estimators with $\lambda$-strongly convex loss functions have uniform stability of order $\mathcal{O}\left(\frac{1}{\lambda n}\right)$~\cite{bousquet2002stability}. This fundamental result then gives rise to the $\ell_2$-norm regularized ERM which introduces a penalty term $\frac{\lambda}{2}\|w\|^2$ to the convex loss with optimal choice $\lambda=\mathcal{O}(n^{-1/2})$ to balance empirical loss and generalization gap~\cite{shalev2009stochastic,feldman2018generalization,feldman2019high}.

Our analysis also relies on the conditions of Restricted Strong Convexity (RSC) and Restricted Lipschitz Smoothness (RLS) which extend the concept of strong convexity and smoothness to the analysis of sparsity recovery methods~\cite{bahmani2013greedy,blumensath2009iterative,jain2014iterative,yuan2018gradient}.
\begin{definition}[Restricted Strong Convexity and Restricted Lipschitz Smoothness]\label{def:strong_smooth}
For any sparsity level $1\le s \le p$, we say a function $f$ is restricted $\mu_s$-strongly convex and $L_s$-smooth if there exist $\mu_s, L_s > 0$ such~that
\[
\frac{\mu_s}{2}\|w-w'\|^2 \le f(w) - f(w') - \langle \nabla f(w'), w -w'\rangle \le \frac{L_s}{2}\|w - w'\|^2, \quad \forall\|w - w'\|_0\le s.
\]
Particularly, we say $f$ is $L$-smooth ($\mu$-strongly convex) if $f$ is $L_p$-smooth ($\mu_p$-strongly convex).
\end{definition}
The ratio number $L_s/\mu_s$ will be referred to as \emph{restricted strong condition number} in this paper. By definition we have $L_s \le L_{s'}$ and $\mu_s \ge \mu_{s'}$ for all $s\le s'$. We say that a function $f$ is $G$-Lipschitz over $\mathcal{W}$ if $|f(w) - f(w')|\le G\|w - w'\|$ for all $w, w'\in \mathcal{W}$. We denote $[p]=\{1,...,p\}$. The following basic assumptions will be made in different combinations in our theoretical analysis.
\begin{assumption}\label{assump:lipschitz}
The convex loss function $\ell$ is $G$-Lipschitz continuous with respect to its first argument and $\ell(\cdot;\xi)\le M$ for all $\xi \in \mathcal{X}$.
\end{assumption}

\begin{assumption}\label{assump:strongly_convex}
The empirical risk $F_S$ is $L_{4k}$-smooth and $\mu_{4k}$-strongly convex with probability at least $1- \delta'_n$ over sample $S$ for some $\delta'_n \in (0,1)$.
\end{assumption}

\begin{assumption}\label{assump:sparisty_relaxation}
Consider $\bar w =\argmin_{\|w\|_0\le \bar k} F(w)$ and set the sparsity level  $k\ge \frac{32 L_{4k}^2}{\mu_{4k}^2}\bar k$ for IHT.
\end{assumption}

\begin{assumption}\label{assump:strong_convex_pop}
The population risk $F$ is $\rho_k$-strongly convex and without loss of generality $\|w\|\le 1, \forall w \in \mathcal{W}$.
\end{assumption}
\begin{remark}
Assumption~\ref{assump:lipschitz} on Lipschitz and bounded loss function is standard in the uniform stability and generalization analysis of ERM~\cite{bousquet2020sharper,bousquet2002stability}. The RSC/RLS conditions in Assumptions~\ref{assump:strongly_convex} and the sparsity relaxation conditions in Assumption~\ref{assump:sparisty_relaxation} are conventionally used in the state-of-the-art convergence analysis of IHT without imposing RIP-type conditions~\cite{jain2014iterative,yuan2018gradient}. Assumption~\ref{assump:strong_convex_pop} is required for establishing fast rate of convergence for IHT via the local Rademacher complexity theory~\cite{bartlett2005local}.
\end{remark}
\subsection{A uniform-stability induced risk bound}
\label{ssect:stability_induced_bound}
We first analyze the excess risk of IHT based on the uniform stability of strongly convex ERM. In order to make sure that the output $w_{S,k}^{(T)}$ \emph{at the end of iteration} has uniform stability, we propose to slightly modify it as $\tilde w_{S,k}^{(T)}$ which just minimizes $F_S$ over the support of $\supp(w_{S,k}^{(T)})$, i.e.,
\[
\tilde w_{S,k}^{(T)} :=\argmin_{w\in \mathcal{W}} F_{S}(w) \quad \st \supp(w) =\supp(w_{S,k}^{(T)}).
\]
Unless otherwise stated, in what follows we will work on the above variant of IHT and assume that $w^{(0)}_{S,k}=0$. In order to avoid RIP-type conditions which are hard to be fulfilled in high-dimensional statistical settings wherein pairs of variables can be arbitrarily correlated, we resort to the techniques developed in~\cite{jain2014iterative,shen2017tight} to analyze IHT under proper sparsity level relaxation conditions. The following result is our first main result on the sparse excess risk bound of IHT in the considered setup.

\begin{theorem}\label{thrm:uniform_stability_iht}
Suppose that Assumptions~\ref{assump:lipschitz},~\ref{assump:strongly_convex},~\ref{assump:sparisty_relaxation} hold. Set the step-size $\eta = \frac{2}{3L_{4k}}$. For any $\delta \in (0, 1-\delta'_n)$, with probability at least $1-\delta-\delta'_n$ over the random draw of sample set $S$, after sufficiently large $T\ge\mathcal{O}\left(\frac{L_{4k}}{\mu_{4k}}\log \left(\frac{nM}{k\log(n)\log(p/k)}\right)\right)$ rounds of IHT iteration, the $\bar k$-sparse excess risk of IHT is upper bounded by
\[
F(\tilde w^{(T)}_{S,k}) - F(\bar w) \le \mathcal{O}\left(\frac{G^{3/2}M^{1/4}}{\mu_{4k}^{3/4}}\sqrt{\frac{\log(n)(\log(1/\delta)+k\log(p/k))}{n}} + M\sqrt{\frac{\log(1/\delta)}{n}} \right).
\]
\end{theorem}
\begin{proof}[Proof in sketch]
The basic idea is to show that a nearly identical bound holds for ERM restricted over any fixed supporting set of size $k$ and that bound can be extended to IHT in light of Lemma~\ref{lemma:convergence_iht} (in Appendix~\ref{apdsect:auxiliary lemmas}) and union probability. More precisely, for a given feature index set $J\subseteq [p]$ with $|J|=k$, we first establish a generalization gap bound for the restrictive estimator over $J$ defined by $w_{S\mid J} := \argmin_{\supp(w)\subseteq J} F_S(w)$. Since $F_S$ is only assumed to have strong convexity over $J$ with high probability, $w_{S\mid J}$ is not necessarily uniformly stable. To handle this issue, we propose to alternatively study an $\ell_2$-regularized variant of $w_{S\mid J}$ defined by
\[
w_{\lambda,S\mid J} := \argmin_{\supp(w)\subseteq J} \left\{F_{\lambda,S}(w):= F_S(w) + \frac{\lambda}{2}\|w\|^2\right\},
\]
which has uniform stability for any $\lambda>0$. Then according to the result from~\cite[Corollary 8]{bousquet2020sharper} its generalization gap is upper bounded by $\mathcal{\tilde O}\left(\frac{\log(n)}{\lambda n} + \frac{1}{\sqrt{n}}\right)$. The next key step is to bound the discrepancy between $w_{S\mid J}$ and $w_{\lambda, S\mid J}$ as $\|w_{S\mid J} - w_{\lambda, S\mid J}\| \le \mathcal{O}\left(\frac{\lambda}{\mu_k + \lambda}\right)$ in view of the (high probability) restricted strong convexity of $F_S$, which consequently indicates that the generalization guarantee of $w_{\lambda, S\mid J}$ can be handed over to $w_{S\mid J}$ with a small overhead of $\mathcal{O}\left(\frac{\lambda}{\mu_k + \lambda}\right)$. Under optimal selection of $\lambda$, applying union probability arguments over all the possible $J$ yields a generalization gap bound for $\ell_0$-ERM. The final step is to show, according to Lemma~\ref{lemma:convergence_iht}, that such a generalization gap bound of $\ell_0$-ERM leads to the desired sparse excess risk bound of IHT after sufficient iteration with proper sparsity relaxation. A full proof of this result is provided in Appendix~\ref{apdsect:proof_uniform_stability_iht}.
\end{proof}
\begin{remark}
Theorem~\ref{thrm:uniform_stability_iht} shows that under proper relaxation of sparsity level, the $\bar k$-sparse excess risk of IHT converges at a rate of $\mathcal{\tilde O}\left(n^{-1/2}\sqrt{k\log(n)\log(p/k)} \right)$, which matches those of the $\ell_0$-penalized binary prediction estimators~\cite{chen2018best,chen2020binary} up to logarithmic factors. Let $w^*=\argmin_{w\in \mathcal{W}} F(w)$ be the global  minimizer without sparsity constraint. Such a sparse excess risk bound immediately gives arise to an oracle inequality in terms of $w^*$:
\[
F(\tilde w^{(T)}_{S,k}) - F(w^*) \le \min_{\|w\|_0\le \bar k} (F(w) - F(w^*)) + \mathcal{\tilde O}\left(\sqrt{\frac{k\log(n)\log(p/k)}{n}} \right).
\]
\end{remark}

\textbf{Implication for sparse logistic regression.} Let us substantialize Theorem~\ref{thrm:uniform_stability_iht} to binary logistic regression model with loss function $\ell(w;\xi) = \log(1+\exp(-2y w^\top x))$ at a labeled data sample $\xi=(x,y)\in \mathbb{R}^p \times \{-1,1\}$. Given a set of $n$ independently drawn data samples $\{(x_i,y_i)\}_{i=1}^n$, sparse logistic regression learns the parameters so as to minimize the logistic loss function under sparsity constraint:
\[
\min_{\|w\|_0\le k} F_S(w)= \frac{1}{n}\sum_{i=1}^n \log(1+\exp(-2y_i w^\top x_i)).
\]
Let $X=[x_1,...,x_n]\in\mathbb{R}^{d\times n}$ be the design matrix and $s(z)= \frac{1}{1+\exp(-z)}$ be the sigmoid function. It can be shown that $\nabla F_S(w) = X a(w)/n$ in which the vector $a(w)\in \mathbb{R}^n$ is given by $[a(w)]_i=-2y_i(1-s(2y_iw^\top x_i))$, and the Hessian $\nabla^2 F_S(w) = X \Lambda(w) X^\top / n$ where $\Lambda(w)$ is an $n\times n$ diagonal matrix whose diagonal entries are $[\Lambda(w)]_{ii}=4s(2y_iw^\top x_i)(1-s(2y_iw^\top x_i))$. Then we have the following corollary as an application of Theorem~\ref{thrm:uniform_stability_iht} to the above sparsity-constrained logistic regression.

\begin{corollary}\label{corol:generalization_barw_logisticreg}
Assume that $x_i$ are i.i.d. zero-mean sub-Gaussian distribution with covariance matrix $\Sigma\succ 0$ and $\Sigma_{jj}\le \frac{\sigma^2}{32}$. Suppose that $\|x_i\|\le 1$ for all $i$ and $\mathcal{W}\subset \mathbb{R}^p$ is bounded by $R$. Then there exist universal constants $c_0, c_1>0$ such that when $n \ge \frac{4kc_1 \log (p)}{\lambda_{\min}(\Sigma)}$, for any $\delta \in (0, 1- \exp\{-c_0 n\})$, with probability at least $1 - \delta - \exp\{-c_0 n\}$ the sparse excess risk of IHT is upper bounded by
\[
F(\tilde w^{(T)}_{S,k}) - F(\bar w) \le \mathcal{O}\left(\frac{\exp(R)}{\lambda^{3/4}_{\min}(\Sigma)}\sqrt{\frac{\log(n)(\log(1/\delta)+k\log(p/k))}{n}} + R\sqrt{\frac{\log(1/\delta)}{n}}\right)
\]
after sufficiently large rounds of iteration, i.e.,
\[
T\ge\mathcal{O}\left(\frac{\exp(R)}{\lambda_{\min}(\Sigma)}\log \left(\frac{n R}{k\log(n)\log(p/k)}\right)\right).
\]
\end{corollary}

\subsection{Sharper bound via support recovery stability analysis}
\label{ssect:sharper_analysis_support_stability}
In addition to the previous analysis built largely on conventional stability theory, we further tailor a support recovery stability theory for IHT aiming at improving upon the previous bound by removing its dependency on sparsity level $k$. For a vector $w\in \mathbb{R}^p$, we denote $[w]_{(j)}$ its entry with $j$-th largest absolute value such that $|[w]_{(1)}| \ge |[w]_{(2)}| \ge...\ge |[w]_{(p)}|$. We begin by introducing the following concept of \emph{hard-thresholding stability} which quantifies the sensitivity of the hard-thresholding operation to entry-wise perturbation.
\begin{definition}[Hard-Thresholding Stability]
For a vector $w\in \mathbb{R}^p$ and given $k\in [p]$, we say $w$ is $\varepsilon_{k}$-hard-thresholding stable for some $\varepsilon_{k}>0$ if and only if $|[w]_{(k)}| \ge |[w]_{(k+1)}| + \varepsilon_{k}$.
\end{definition}
Clearly, if $w$ is $\varepsilon_{k}$-hard-thresholding stable, then $\mathrm{H}_k(w)$ would be unique and $\supp\left(\mathrm{H}_k(w)\right)=\supp\left(\mathrm{H}_k(w+\delta_w)\right)$ provided that the perturbation is sufficiently small such that $\|\delta_w\|_\infty<\varepsilon_k/2$. In other words, the larger $\varepsilon_k$ is, the stabler the hard-thresholding operation will be with respect to the preserved top $k$ supporting set. On top of this, we next introduce the following concept of \emph{iterative-hard-thresholding stability} which basically characterizes the stability of the IHT algorithm when applied to a (deterministic) function.
\begin{definition}[Iterative-Hard-Thresholding Stability]\label{def:iht_stability}
For a given deterministic and differentiable function $F$, $k \in [p]$, $T\in \mathbb{Z}^+$ and $k$-sparse vector $w^{(0)}\in \mathbb{R}^p$, let $\{w^{(t)}\}_{t=1}^T$ be the sequence generated by invoking IHT on $F$ with step-size $\eta$ and initialization $w^{(0)}$. Then we say $F$ is $(\varepsilon_{k},\eta, T, w^{(0)})$-IHT stable if $w^{(t-1)} - \eta \nabla F(w^{(t-1)})$, $\forall t\in [T]$ is $\varepsilon_{k}$-hard-thresholding stable.
\end{definition}
\begin{remark}
By definition, if $F$ is $(\varepsilon_{k},\eta, T, w^{(0)})$-IHT stable, then for each $t \in [T]$, $w^{(t)}=\mathrm{H}_k\left(w^{(t-1)}-\eta \nabla F(w^{(t-1)})\right)$ is unique. That is, the $k$-sparse solution sequence $\{w^{(t)}\}_{t=1}^T$ generate by IHT is unique.
\end{remark}
\textbf{An illustrating example.} To gain some intuition of IHT-stability, we consider the sparse linear regression model with quadratic loss function $\ell(w; x_i, y_i)=\frac{1}{2}(y_i-w^{\top}x_i)^2$. The feature points $\{x_i\}_{i=1}^n$ are sampled from standard multivariate Gaussian distribution. Given a model parameter $\tilde w \in \mathbb{R}^p$, the responses $\{y_i\}_{i=1}^n$ are generated according to a linear model $y_i=\tilde{w}^{\top}x_i+\varepsilon_i$ with a random Gaussian noise $\varepsilon_i\sim \mathcal{N}(0,\sigma^2)$. In this case, the population risk function can be expressed in close form as
\begin{equation}\label{equat:linear_regression_close_form}
F(w) = \frac{1}{2}\|w - \tilde w\|^2 + \frac{\sigma^2}{2}.
\end{equation}
{Suppose that $\tilde w $ has relatively weaker signal strength outside the supporting set $J$ of its top-$\bar k$ entries such that $|[\tilde w]_{(\bar k)}| \ge |[\tilde w]_{(\bar k+1)}| + \bar \varepsilon$ for some $\tilde \varepsilon >0$. Then given $w^{(0)}=0$ and any step-size $\eta\in (0,1)$, it can be verified that the population risk $F$ is $(\varepsilon_{\bar k},\eta, T, w^{(0)})$-IHT stable with $\varepsilon_{\bar k}=\eta \bar \varepsilon$. Indeed, based on the close-form expression~\eqref{equat:linear_regression_close_form}, we can prove by induction that for all $t\ge 1$,
\[
w^{(t)} = \mathrm{H}_{\bar k} \left(w^{(t-1)}-\eta \nabla F(w^{(t-1)})\right) = \mathrm{H}_{\bar k} \left((1- (1-\eta)^t) \tilde w_J + \eta \tilde w_{\overline{J}}\right) = (1- (1-\eta)^t) \tilde w_J,
\]
where in the last equality we have used the fact $1- (1-\eta)^t \ge \eta$ when $t\ge 1$. This then implies the desired IHT-stability of $F$. Therefore in this example, for a fixed $\eta\in (0,1)$, the stability strength $\varepsilon_{\bar k}=\eta\bar \varepsilon$ is controlled by the signal strength gap $\bar\varepsilon$.}

The following theorem is our main result on the risk bound of IHT given that the population risk function $F$ has IHT-stability up to the desired number of iteration.

\newpage

\begin{theorem}\label{thrm:uniform_stability_strong_iht}
Suppose that Assumptions~\ref{assump:lipschitz},~\ref{assump:strongly_convex},~\ref{assump:sparisty_relaxation} hold. Consider running $T$ steps of IHT iteration over $F_S$ with step-size $\eta= \frac{2}{3L_{4k}}$ from initialization $w_{S,k}^{(0)}=0$. Assume that the population risk function $F$ is $(\varepsilon_{k},\eta, T, 0)$-IHT stable. For any $\delta \in (0, 1-2\delta'_n)$, if $n\ge \frac{2G^2(L_{4k}+\mu_{4k})^2\log(2pT/\delta)}{L_{4k}^2\mu^2_{4k}\varepsilon^2_k}$ and $T\ge\mathcal{O}\left(\frac{L_{4k}}{\mu_{4k}}\log \left(\frac{nM}{\log(n)\log(1/\delta)}\right)\right)$, then with probability at least $1-\delta-2\delta'_n$ over the random draw of sample set $S$, the $\bar k$-sparse excess risk of IHT is upper bounded as
\[
F(\tilde w^{(T)}_{S,k}) - F(\bar w) \le \mathcal{O}\left(\frac{G^{3/2}M^{1/4}}{\mu_{4k}^{3/4}}\sqrt{\frac{\log(n)\log(1/\delta)}{n}} + M\sqrt{\frac{\log(1/\delta)}{n}} \right).
\]
\end{theorem}
\begin{proof}[Proof in sketch]
The key proof ingredient is to construct a \emph{hypothetical sequence} $\{w^{(t)}\}_{t=1}^T$ generated by applying $T$ rounds of IHT iteration to (unknown) $F$ with the considered initialization $w^{(0)}$ and step-size $\eta$. Given that $F$ is $(\varepsilon_{k},\eta, T, w^{(0)})$-IHT stable, we can show in Lemma~\ref{lemma:emp_iht_stability} (see Appendix~\ref{apdsect:proof_uniform_stability_strong_iht}) that the actual sequence $\{w_{S,k}^{(t)}\}_{t=1}^T$ generated by IHT invoked to the empirical risk $F_S$ (with any fixed $S$) satisfies $\supp\left(w^{(t)}_{S,k}\right) = \supp\left(w^{(t)}\right), \forall t\in [T]$ provided that $n$ is sufficiently large as assumed. Particularly, we have $\supp\left(w^{(T)}_{S,k}\right) = \supp\left(w^{(T)}\right)$ which is a fixed deterministic index set of size $k$. Then using the proof arguments of Lemma~\ref{lemma:support_stability} (in Appendix~\ref{apdsect:proof_uniform_stability_iht}) yields a desirable high probability generalization gap bound for $\tilde w_{S,k}^{(T)}$. The excess risk bound can be proved in light of that generalization gap bound and the convergence result in Lemma~\ref{lemma:convergence_iht}. A full proof of this theorem is provided in Appendix~\ref{apdsect:proof_uniform_stability_strong_iht}.
\end{proof}
\begin{remark}
The $\mathcal{O}(n^{-1/2}\sqrt{\log(n)})$ bound established in Theorem~\ref{thrm:uniform_stability_strong_iht} is not relying on sparsity level $k$ and in this sense tighter than that in Theorem~\ref{thrm:uniform_stability_iht}, yet under more stringent conditions on the IHT stability of the population risk $F$.
\end{remark}

Before closing the analysis in this section, we briefly comment on an alternative way to guarantee the support recovery stability of IHT based on the exact support recovery result from~\cite[Theorem 8]{yuan2018gradient} under additional strong-signal conditions.  Indeed, the result in~\cite[Theorem 8]{yuan2018gradient} essentially suggests that if the minimal (in absolute value) non-zero entry of $\bar w$ is significantly larger than $\sqrt{k}\|\nabla F(\bar w)\|_\infty + \sqrt{\frac{k\log(p)}{n}}$, then the exact support recovery $\supp(w^{(t)}_{S,k}, \bar k)=\supp(\bar w)$ holds with high probability after sufficient iteration. Given that the support recovery is stable, we can show that an almost identical risk bounds to that of Theorem~\ref{thrm:uniform_stability_strong_iht} is valid for IHT with an additional $\bar k$-sparse hard thresholding operation over its output. Such a result, however, explicitly requires the knowledge of $\bar k$ for postprocessing which is typically unavailable in realistic problems. Nevertheless, as we will show in the next section that such a way of support recovery stability analysis actually turns out to be powerful for deriving fast convergence rates of IHT for strongly convex optimization~problems.

\newpage

\section{Stronger Risk Bounds with Fast Rates}
\label{sect:stability_risk_bounds_fast}

In consistency with statistical learning theory~\cite{bartlett2005local,srebro2010smoothness,foster2019statistical}, we regard the $\mathcal{\tilde O}(n^{-1/2})$ rates of convergence established so far as \emph{slow} rates in terms of sample size. In the absence of sparsity constraint, it is well known that convergence rates of order $\mathcal{\tilde O}(n^{-1})$ are possible for finite dimensional strongly convex function classes, for instance, via local Rademacher complexities~\cite{bartlett2005local,koltchinskii2006local}. Inspired by such type of \emph{fast} rates for strongly convex dense ERM, we further show in this section that the $\mathcal{\tilde O}(n^{-1})$ sparse excess risk bounds can also be derived for IHT under additional regularity conditions on signal strength of the target sparse minimizer. Moreover, specially for well-specified sparse learning models such as sparse generalized linear models, we show that the $\mathcal{\tilde O}(n^{-1})$-type of fast rates can be derived much more directly based on the classical parameter estimation error bounds of IHT.

\subsection{Fast rates under strong-signal conditions}

In what follows, we denote $w_{\min}:=\min_{i \in \supp(w)} |w_i|$ as the smallest (in modulus) non-zero entry of a sparse vector $w$.

\begin{theorem}\label{thrm:generalizaion_fast_rate}
Suppose that Assumptions~\ref{assump:lipschitz},~\ref{assump:strongly_convex},~\ref{assump:sparisty_relaxation},~\ref{assump:strong_convex_pop} hold. Set the step-size $\eta = \frac{2}{3L_{4k}}$. For any given $\delta\in(0,1)$, assume that $\delta'_n \le \frac{\delta}{4}$ for large enough $n$ and
\[
\bar w_{\min}> \frac{2\sqrt{2k}\|\nabla F(\bar w)\|_\infty}{\mu_{4k}} + \frac{3G}{\mu_{4k}} \sqrt{\frac{k\log(4p/\delta)}{n}}.
\]
Then after sufficiently large $T\ge\mathcal{O}\left(\frac{L_{4k}}{\mu_{4k}}\log \left(\frac{n\mu_{4k}}{kG\log(n)\log(p/k)}\right)\right)$ rounds of IHT iteration, the following $\bar k$-sparse excess risk bound holds with probability at least $1-\delta$ over the random draw~of~$S$:
\[
F(\tilde w^{(T)}_{S,k}) - F(\bar w) \le \mathcal{O}\left(\frac{G^2(\log^{3}(\rho_k n)+ \log(ep/k))}{\rho_k} \left(\frac{k}{n}\right)  + \frac{\log(1/\delta)}{n} \right).
\]
\end{theorem}
\begin{proof}[Proof in sketch]
Let us denote $w^*_J = \argmin_{\supp(w) \subseteq J} F(w)$ for any fixed indices set $J\subseteq[p]$ with $|J|=k$. A core observation here is that under the strong-signal condition on $\bar w_{\min}$, we can show via Lemma~\ref{lemma:support_recovery_bar_w} (in Appendix~\ref{apdsect:proof_theorem_generalizaion_fast_rate}) that $\supp(\bar w)\subseteq \tilde J:=\supp(w^{(T)}_{S,k})$ holds with high probability, and thus does $F(w^*_{\tilde J})\le F(\bar w)$. As another key ingredient, we then show through Lemma~\ref{lemma:fast_key_1} (in Appendix~\ref{apdsect:proof_theorem_generalizaion_fast_rate}) that $\sup_{J\subseteq[p], |J|=k} \left\{F(w_{S\mid J}) - F(w^*_J)\right\}$ is uniformly upper bounded as $\mathcal{\tilde O}(k/n)$ with high probability. In view of this supporting-set-wise uniform excess risk bound, the desired bound follows directly by noting $F(w^{(T)}_{S,k}) - F(\bar w) \le F(w^{(T)}_{S,k}) - F(w^*_{\tilde J}) = F(w_{S\mid \tilde J}) - F(w^*_{\tilde J})$. A full proof of this result is provided in Appendix~\ref{apdsect:proof_theorem_generalizaion_fast_rate}.
\end{proof}

\newpage

\begin{remark}
Consider the well-specified setting where $\nabla F(\bar w)=0$, i.e., the minimizer of the population risk is truly sparse. In this case, under the signal-strength condition $\bar w_{\min}= \tilde\Omega\left(\sqrt{k/n}\right)$, Theorem~\ref{thrm:generalizaion_fast_rate} suggests that the sparse excess risk bound of IHT decays as fast as $ \mathcal{\tilde O}(k/n)$ with high probability. A benefit of the result in Theorem~\ref{thrm:generalizaion_fast_rate} is that it allows for misspecified sparse models. More precisely, even if the risk $F$ does not have zero gradient at $\bar w$, the sparse excess risk of IHT can still converge as fast as $\mathcal{\tilde O}(k/n)$ provided that $\bar w_{\min}$ significantly outweighes $\tilde\Omega(\sqrt{k}\|\nabla F(\bar w)\|_\infty+\sqrt{k/n})$.
\end{remark}

{It is noteworthy that the RSC/RLS parameters are not involved in the risk bounds of Theorem~\ref{thrm:generalizaion_fast_rate}. As a matter of fact, in this setting the RSC/RLS conditions are imposed only to guarantee the sparsity recovery and stability of IHT. Therefore, the RSC/RLS conditions are not essential to our sparse excess risk analysis provided that the convergence and stability of IHT can alternatively be guaranteed without these conditions. Recently, it has been shown that the IHT-style methods can be extended to non-smooth $\ell_0$-ERM problems via smoothing approximation techniques~\cite{wu2021smoothing}. For gradient descent methods without hard thresholding, it is clear that strong convexity is not necessary to guarantee convergence and stability. However, to our knowledge it still remains an open problem to guarantee the sparsity recovery of IHT without assuming RSC-type conditions.}

\subsection{Fast rates for well-specified sparse learning models}
\label{apdsect:fast_rates_well_specified}

The sparse excess risk bounds derived so far are essentially for misspecified sparse learning models. In this subsection, we further study the risk bounds of IHT in well-specified scenarios where the data is assumed to be generated according to a truly sparse model. Such a statistical treatment is conventional in the theoretical analysis of high-dimensional sparsity recovery approaches~\cite{agarwal2012fast,mei2018landscape,yuan2018gradient}. More specifically, we assume that there exists a $k$-sparse parameter vector $\bar w$ such that, roughly speaking, the population risk function is minimized exactly at $\bar w$ with $\nabla F(\bar w)=0$. Formally, we impose the following assumption on the loss function which basically requires the gradient of loss at $\bar w$ obeys a light tailed distribution.

\begin{assumption}[Sub-Gaussian gradient at the true model]\label{assump:gradient_sub_gaussian}
For each $j\in\{1,...,p\}$, we assume that $\nabla_j \ell(\bar w;\xi)$ is $\sigma^2$-sub-Gaussian with zero mean, namely, $\mathbb{E}_\xi[\nabla_j \ell(\bar w;\xi)]=0$ and there exists a constant $\sigma>0$ such that for any real number $\tau$,
\[
\mathbb{E}_\xi\left[\exp\left\{\tau(\nabla_j \ell(\bar w;\xi))\right\}\right] \le \exp\left\{\frac{\sigma^2\tau^2}{2}\right\}.
\]
\end{assumption}
\begin{remark}
The zero-mean assumption directly implies $\nabla F(\bar w) = 0$. As we will show shortly, this assumption can be fulfilled by the widely used linear regression and logistic
regression models.
\end{remark}
As a side contribution of this work, we present in the following theorem a sharper excess risk bound of IHT for well-specified sparse learning models under less stringent conditions. A proof of this theorem is deferred to Appendix~\ref{apdsect:proof_white_box}.

\begin{theorem}\label{thrm:generalization_barw_iht}
Assume that $\bar w$ is a $\bar k$-sparse vector satisfying Assumption~\ref{assump:gradient_sub_gaussian}. Suppose that Assumptions~\ref{assump:strongly_convex} and~\ref{assump:sparisty_relaxation} hold, and he population risk $F$ is $L$-smooth. Then for any $\delta \in (0, 1- \delta'_n)$, where $\delta'_n\in(0,1)$, and any $\epsilon>0$, IHT with step-size $\eta = \frac{2}{3L_{4k}}$ and sufficiently large $T \ge\mathcal{O}\left(\frac{L_{4k}}{\mu_{4k}}\log \left(\frac{n\mu_{4k}}{k\sigma^2\log(p/\delta)}\right)\right)$ rounds of iteration will output $w^{(T)}_{S,k}$ such that the following sparse excess risk bound holds with probability at least $1 - \delta - \delta'_n$ over $S$,
\[
F(w^{(T)}_{S,k}) - F(\bar w) \le \mathcal{O}\left(\frac{L}{\mu_{4k}^2}\left(\frac{k\sigma^2\log (p/\delta)}{n}\right) \right).
\]
\end{theorem}
\begin{remark}
In comparison to the risk bound established in Theorem~\ref{thrm:generalizaion_fast_rate} that allows for misspecified models, the above fast rate of convergence for well-specified models is sharper in the sense that it is not dependence on $\log(n)$-factors and it is valid without needing to assume Lipschitz-loss and strong-signal conditions.
\end{remark}
\begin{remark}
{We comment on the tightness of the excess risk bounds in Theorem~\ref{thrm:generalization_barw_iht} in the minimax sense. For well-specified sparse linear regression models with RSC, it has been shown~\cite{zhang2014lower} that up to logarithmic factors, the $\mathcal{\tilde O} \left(n^{-1}k\log(p)\right)$ squared estimation error bound is minimax optimal for polynomial time sparse estimators such as Lasso and IHT. This immediately implies that the same bound should be minimax optimal for the excess risk of IHT provided that the population function $F$ is strongly convex.}
\end{remark}

We next showcase how to apply the bounds in Theorem~\ref{thrm:generalization_barw_iht} to the widely used sparse linear regression and logistic regression models.

\textbf{Implication for sparse linear regression.} We assume the samples $S=\{x_i, y_i\}$ obey the linear model $y_i = \bar w^\top x_i + \varepsilon_i$ where $\bar w$ is a $k$-sparse parameter vector, $x_i$ are drawn i.i.d. from a zero-mean sub-Gaussian distribution with covariance matrix $\Sigma\succ 0$, and $\varepsilon_i$ are $n$ i.i.d. zero-mean sub-Gaussian random variables with parameter $\sigma^2$. The sparsity-constrained least squares regression model is then written by
\[
\min_{\|w\|_0\le k} F_S(w)= \frac{1}{2n}\sum_{i=1}^n\|y_i - w^\top x_i\|^2.
\]
We present the following corollary as a consequence of Theorem~\ref{thrm:generalization_barw_iht} to the considered linear regression model with bounded design. See Appendix~\ref{apdsect:proof_corol_generalization_barw_linearreg} for its proof.

\begin{restatable}{corollary}{ERMWhiteBoxBoundsLinaer}\label{corol:generalization_barw_linearreg}
Assume that $\varepsilon_i$ are i.i.d. zero-mean $\sigma^2$-sub-Gaussian and $x_i$ are i.i.d. zero-mean sub-Gaussian distribution with covariance matrix $\Sigma\succ 0$ and $\Sigma_{jj}\le 1$. Then there exist universal constants $c_0, c_1>0$ such that when $n \ge \frac{4kc_1 \log (p)}{\lambda_{\min}(\Sigma)}$, for any $\delta \in (0, 1- \exp\{-c_0 n\})$, with probability at least $1 - \delta - \exp\{-c_0 n\}$ the sparse excess risk of IHT with step-size $\eta=\mathcal{O}\left(\frac{1}{\lambda_{\max}(\Sigma)}\right)$ is bounded as
\[
F(w^{(T)}_{S,k}) - F(\bar w) \le \mathcal{O}\left(\frac{\lambda_{\max}(\Sigma)\sigma^2\log (p/\delta)}{\lambda^2_{\min}(\Sigma)}\left(\frac{k}{n}\right)\right)
\]
after $T\ge\mathcal{O}\left(\frac{\lambda_{\max}(\Sigma)}{\lambda_{\min}(\Sigma)}\log \left(\frac{n \lambda_{\min}(\Sigma)}{k\sigma^2\log(p/\delta)}\right)\right)$ rounds of iteration.
\end{restatable}

\textbf{Implication for sparse logistic regression.} Let us further consider a well-specified binary logistic regression model in which the relation between the random feature vector $x \in \mathbb{R}^p$ and its associated random binary label $y \in \{-1,+1\}$ is determined by the conditional probability $\mathbb{P}(y|x; \bar w) = \frac{\exp (2y\bar w^\top x)}{1+\exp (2y\bar w^\top x)}$, where $\bar w$ is a $k$-sparse parameter vector. Then we have the following corollary as an application of Theorem~\ref{thrm:generalization_barw_iht} to this well-specified sparse logistic regression model. A proof of this result is provided in Appendix~\ref{apdsect:proof_corol_generalization_barw_linearreg}.

\begin{corollary}\label{corol:generalization_barw_logisticreg_fast}
Assume that $x_i$ are i.i.d. zero-mean sub-Gaussian distribution with covariance matrix $\Sigma\succ 0$ and $\Sigma_{jj}\le \frac{\sigma^2}{32}$. Suppose that $\|x_i\|\le 1$ for all $i$ and $\mathcal{W}\subset \mathbb{R}^p$ is bounded by $R$. Then there exist universal constants $c_0, c_1>0$ such that when $n \ge \frac{4kc_1 \log (p)}{\lambda_{\min}(\Sigma)}$, for any $\delta \in (0, 1- \exp\{-c_0 n\})$, with probability at least $1 - \delta - \exp\{-c_0 n\}$ the sparse excess risk of IHT with step-size $\eta=\mathcal{O}(1)$ is upper bounded by
\[
F(w^{(t)}_{S,k}) - F(\bar w) \le \mathcal{O} \left(\frac{\exp(R)\sigma^2\log(p/\delta)}{\lambda^2_{\min}(\Sigma)}\left(\frac{k}{n}\right)\right)
\]
after $T\ge\mathcal{O}\left(\frac{\exp(R)}{\lambda_{\min}(\Sigma)}\log \left(\frac{n \lambda_{\min}(\Sigma)}{k\exp(R)\sigma^2\log(p/\delta)}\right)\right)$ rounds of iteration.
\end{corollary}

\section{Discussions}
\label{sect:comparison}

In this section, we discuss the connections and differences between our results established in the previous sections and a number of existing risk bounds for the $\ell_0$-ERM and Lasso-type estimators. For each estimator, we distinguish the comparison in two settings of fast and slow convergence rates respectively.

\subsection{Comparison with the risk bounds for $\ell_0$-ERM}

We begin with comparing our results with existing risk bounds for the $\ell_0$-ERM.

\textbf{Fast rates.} Given that the $\ell_0$-ERM estimator is exactly solved, an essentially $\mathcal{\tilde O}(n^{-1}k\log(p))$ sparse excess risk bound has been established for its output over bounded liner prediction classes~\cite[Example 2]{foster2019orthogonal}. That bound, however, is more of pure theoretical interest than practical usage due to the computational hardness of $\ell_0$-ERM. Contrastingly, Theorem~\ref{thrm:generalizaion_fast_rate} shows that an about the same fast rate of convergence can also be derived for IHT which is computationally tractable and efficient for sparsity recovery.

\textbf{Slow rates.} We comment on the difference between the $ \mathcal{\tilde O} \left(n^{-1/2}\sqrt{k\log(n)\log(p)}\right)$ rate in Theorem~\ref{thrm:uniform_stability_iht} and a comparable result established via uniform concentration bounds~\cite[Theorem 1]{chen2018best} for sparsity constrained binary prediction problems. First, our bound holds for IHT as a computationally tractable estimator for $\ell_0$-ERM while that bound was established for the $\ell_0$-ERM itself. Second, regarding the regularization condition, the result in~\cite[Theorem 1]{chen2018best} requires $p\vee n \gtrsim  k^8$ which could be unrealistic even when $k$ is moderate in real problems. In contrast, our analysis does not impose such fairly stringent conditions on data scale.

\subsection{Comparison with the risk bounds for Lasso estimators}

We further compare the excess risk bounds of IHT to those of the $\ell_1$-regularized ERM (Lasso) estimator~\cite{tibshirani1996regression,wainwright2009sharp}
\[
w^{\ell_1}_{S,\lambda}:=\argmin_{w \in \mathcal{W}} F_S(w)+ \lambda \|w\|_1
\]
which is popularly used as a convex surrogate of the $\ell_0$-ERM~\cite{tibshirani1996regression,wainwright2009sharp} estimator.

\textbf{Fast rates.} For high-dimensional generalized linear models (GLMs), the oracle inequality in~\cite[Theorem 2.1]{van2008high} suggests that if the target solution $w^*=\argmin_{w \in \mathcal{W}} F(w)$ is exactly $k$-sparse, then it holds with high probability that
\[
F(w^{\ell_1}_{S,\lambda}) - F(w^*) \le \mathcal{\tilde O}\left(\frac{k\log(p)}{n}\right)
\]
under $\lambda\asymp n^{-1/2}\sqrt{\log(p)}$. {Also for the well-specified sparse GLMs, similar fast rates of convergence can be implied under RSC/RLS conditions by the parameter estimation error bounds established in~\cite[Corollary 2]{negahban2012unified}. In Theorem~\ref{thrm:generalization_barw_iht}, we have shown that the $\mathcal{\tilde O}(n^{-1}k\log(p))$ rate is also possible for IHT in well-specified models with sub-Gaussian noises, which is applicable to models beyond GLMs.} In comparison to these fast rates for well-specified sparsity models, the fast rate established in Theorem~\ref{thrm:generalizaion_fast_rate} is different in the sense that its corresponding target solution is certain sparse minimizer of $F$ instead of its global minimizer, and thus the result applies to misspecified sparsity models.

\textbf{Slow rates.} For well-specified linear regression models with the $\ell_1$-norm of parameter vector upper bounded by $K$, it has been shown in~\cite{chatterjee2013assumptionless} that the expected excess risk of a constrained Lasso estimator scales as $\mathcal{\tilde O}(n^{-1/2}K^2\sqrt{\log(p)})$ under mild conditions on the design matrix. To compare it with the $\mathcal{\tilde O}(n^{-1/2}\sqrt{k\log(n)\log(p)})$ ($\mathcal{\tilde O}(n^{-1/2}\sqrt{\log(n)})$ ) bound of IHT established in Theorem~\ref{thrm:uniform_stability_iht} (Theorem~\ref{thrm:uniform_stability_strong_iht}), we remark that 1) these two rates are comparable (up to logarithmic factors) when $K^2\asymp \sqrt{k}$ ($K\asymp 1$); 2) the former holds in expectation while the latter (ours) holds in high probability; and 3) most importantly, our bound is applicable to a broader range of learning problems beyond well-specified sparsity models, yet at the price of imposing more stringent assumptions on the risk function.

\section{Simulation Study}
\label{sect:simulation_study}

In this section, we carry out a set of numerical experiments on synthetic sparse logistic regression and least squared regression tasks to verify the IHT generalization theory presented in Section~\ref{sect:stability_risk_bounds} and Section~\ref{sect:stability_risk_bounds_fast}. Throughout our numerical study, we initialize $w^{(0)} = 0$ for IHT and replicate each individual experiment $10$ times over the random generation of training data for generalization performance evaluation.

\subsection{On the Scaling Law of Sparsity Level}
\label{apdsect:simulation_study_theorem_stability_iht}

We first demonstrate the scaling law of sparsity level for the sparse excess risk of IHT in a well-specified sparse logistic regression task.

\textbf{Experiment setup.} We consider the binary logistic regression model with loss function $\ell(w; x_i, y_i) = \log\left(1+\exp(-y_{i}w^\top x_{i})\right)$. In this set of simulation study, each data feature $x_i$ is sampled from standard multivariate Gaussian distribution and its binary label $y_i \in \{-1,+1\}$ is determined by the conditional probability $\mathbb{P}(y_i|x_i; \bar w) = \frac{\exp (2y_i \bar w^\top x)}{1+\exp (2y_i \bar w^\top x_i)}$ with a $\bar k$-sparse parameter vector $\bar{w}$. In such a well-specified setting, we test with feature dimension $p=1000, \bar k=50$ and aim to show the impact of varying sparsity level $k/\bar k\in [1, 4]$ and sample size $n/p\in \{2, 5,10\}$ on the actual sparse excess risk of IHT. Since for logistic loss the population risk function $F$ has no close-form expression, we approximate the population value $F(w)$ by its empirical version with sufficient sampling. In order to compute the excess risk, we need to estimate the optimal population risk which in view of the proof of Corollary~\ref{corol:generalization_barw_logisticreg} is given by $\min_{\|w\|_0\le k} F(w)=F(\bar w)$ for any $k \ge \bar k$.

\textbf{Numerical results.} The evolving curves of sparse excess risk as functions of sparsity level under different sample sizes are shown in Figure~\ref{fig:excessrisk_sparsity_white_logistic}. For each fixed sample size $n$, we can see that the sparse excess risk scales roughly linearly with respect to $k$. For each fixed sparsity level $k$, the sparse excess risk decreases as $n$ increases. These observations are consistent with the fast rates established in Theorem~\ref{thrm:generalizaion_fast_rate} and Theorem~\ref{thrm:generalization_barw_iht} which are applicable to the considered well-specified sparse binary logistic regression problem.

\begin{figure}[h!]
\centering
\includegraphics[width=4.7in]{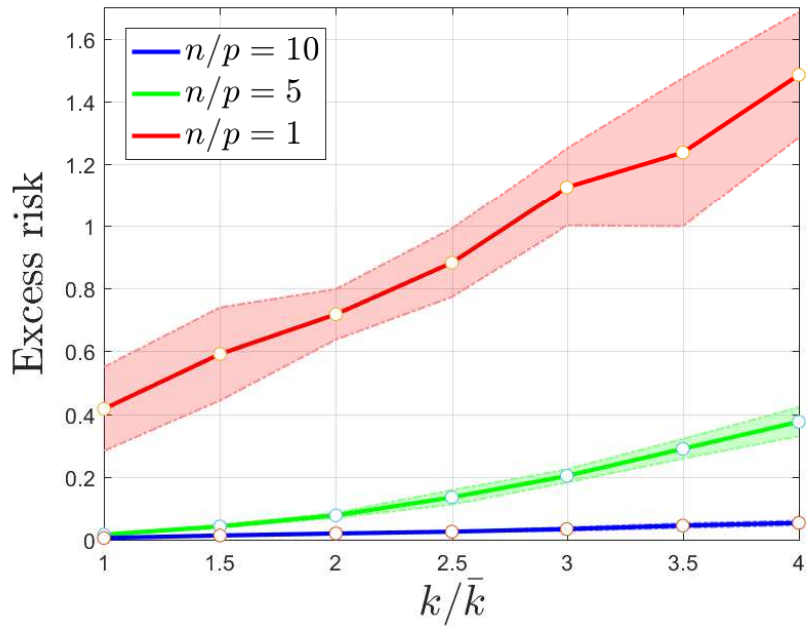}
\caption{Sparse logistic regression: scaling of sparse excess risk with respect to the sparsity level of IHT.}
\label{fig:excessrisk_sparsity_white_logistic}
\end{figure}

\subsection{On the IHT-Stability Theory}
\label{apdsect:simulation_study_theorem_strong_iht}

Further, we show some numerical evidences to support our IHT stability and generalization theory as presented in Theorem~\ref{thrm:uniform_stability_strong_iht}. The main message conveyed by this result is that when sample size is sufficiently large, the IHT-stability of the population risk $F$ plays an important role for obtaining tighter risk bounds. To confirm this theoretical prediction, we consider the linear regression example presented in~\eqref{equat:linear_regression_close_form}. {In this case, given $w^{(0)}=0$ and any step-size $\eta\in (0,1)$, we have shown in Section~\ref{ssect:sharper_analysis_support_stability} that the population risk $F$ is $(\varepsilon_{\bar k},\eta, T, w^{(0)})$-IHT-stable with $\varepsilon_{\bar k}=\eta\bar \varepsilon$ where $\bar\varepsilon$ represents the gap between $|[\tilde w]_{(\bar k)}|$ and $|[\tilde w]_{(\bar k+1)}|$. In our experiment, based on a fixed zero-mean Gaussian vector $\hat w$ with top-$\bar k$ index set $J$, we construct $\tilde w$ as $[\tilde w]_j =  [\hat w]_j + \bar\varepsilon \sign ([\hat w]_j) $ for $j\in J$ and $[\tilde w]_j =  [\hat w]_j $ otherwise, such that the gap between the top $\bar k$ entries and the rest ones of $\tilde w$ is at least $\bar\varepsilon$. We test with  feature dimension $p=1000$, $\bar k=100$ and $\eta=0.5$. Figure~\ref{fig:stability_white_linear} shows the evolving curves of excess risk under varying $\bar\varepsilon\in (0,1)$ and $n/p\in \{ 1, 5, 10\}$. For better visualization, the figure is presented in a semi-log layout with y-axis representing the logarithmic scale of excess risk. These curves indicate that for each fixed sample size $n$, better generalization performance can be achieved under relatively larger $\bar\varepsilon$, which is consistent with the condition of Theorem~\ref{thrm:uniform_stability_strong_iht} about the scaling law of sample size $n$ with respect to $\bar\varepsilon$.}

\begin{figure}[h!]
\centering
\includegraphics[width=4.9in]{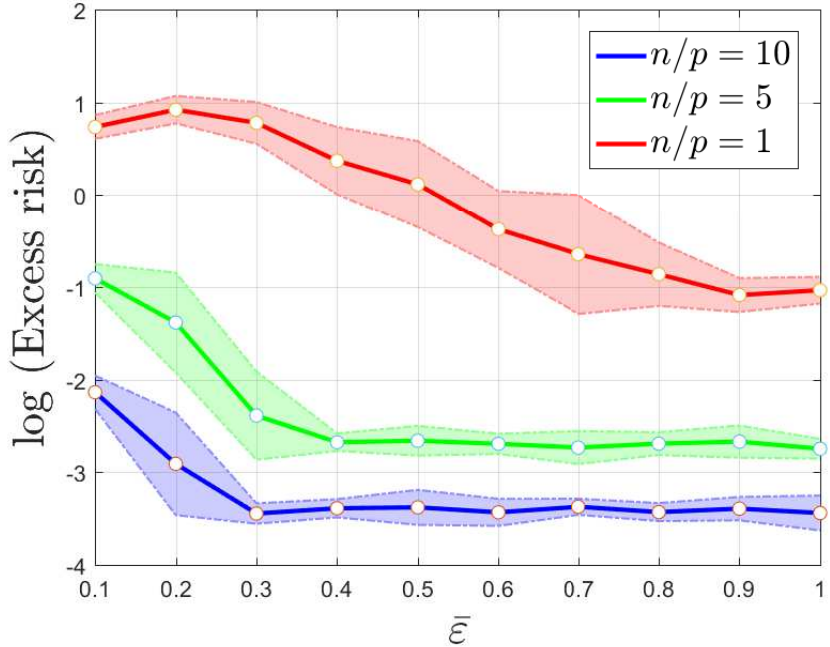}
\caption{Sparse linear regression: impact of IHT-stability on the sparse excess risk of IHT.\label{fig:stability_white_linear}}
\end{figure}

\section{Conclusions}
\label{sect:conclusion}

In this paper, we established a set of novel sparse excess risk bounds for the widely applied IHT method under the notion of unform stability. Specifically, we have shown that the sparse excess risk of IHT converges at the rate of $\mathcal{\tilde O}(n^{-1/2}\sqrt{k\log(n)\log(p)})$ with high probability under natural regularity conditions and the rate can be enhanced to $\mathcal{\tilde O}(n^{-1/2}\sqrt{\log(n)})$ provided that the population risk is stable with respect to IHT iteration. Under additional strong-signal conditions, we further proved faster rates of order $\mathcal{\tilde O}(n^{-1}k(\log^3(n)+\log(p)))$ for strongly convex risk minimization problems. These sparse excess risk bounds immediately give rise to oracle excess risk inequalities of IHT over cardinality constraint. As a side contribution, we have shown that a fast rate of $\mathcal{\tilde O}(n^{-1}k\log(p))$ applies to IHT for well-specified sparse learning with sub-Gaussian noises. These results confirm that the strong excess risk bounds recently established for $\ell_0$-ERM~\cite{chen2018best,chen2020binary,foster2019statistical} can indeed be attained by its approximate solver IHT in a computationally tractable and efficient way.

We expect that the theory developed in this paper will fuel future investigation on the generalization bounds of IHT for non-convex loss functions such as those used in the common practice of deep neural networks pruning~\cite{frankle2019lottery,han2016deep}, yet rarely studied in theory~\cite{sun2019optimization}. In view of the standard $\mathcal{\tilde O}\left(n^{-1/2}\sqrt{p}\right)$ uniform convergence bound for dense models (see, e.g.,~\cite{shalev2009stochastic}), it is more or less straightforward to derive a generalization bound of order $\mathcal{\tilde O} \left(n^{-1/2}\sqrt{k\log(p)}\right)$ for IHT which is applicable to the non-convex regime. In view of the recent progresses achieved towards understanding the benefit of overparametrization for the optimization and generalization of gradient-based deep learning algorithms~\cite{allen2019learning,oymak2019overparameterized}, it is interesting to further study the impact of overparametrization on the generalization performance of IHT-style algorithms for deep learning with sparsity.

\bibliographystyle{plain}
\bibliography{ref}

\appendix

\makeatletter
\renewcommand\theequation{A.\@arabic\c@equation }
\makeatother \setcounter{equation}{0}

\section{Some auxiliary lemmas}
\label{apdsect:auxiliary lemmas}

This section is devoted to presenting a set of preliminary results that are useful in the proof of our main results.

\emph{Generalization bounds for uniformly stable algorithms.} To prove the stability implied risk bounds, we need the following lemma from~\cite[Corollary 8]{bousquet2020sharper} which gives a near-tight generalization error bound for uniformly stable learning algorithms.
\begin{lemma}[Generalization bound implied by uniform stability]\label{lemma:uniform_stability}
Let $A: \mathcal{X}^n \mapsto \mathcal{W}$ be a learning algorithm that has uniform stability $\gamma$  with respect to a loss function $\ell(\cdot;\cdot) \le M$. Then for any $\delta\in(0,1)$, the following generalization bound holds with probability at least $1-\delta$ over $S$:
\[
\left|\mathbb{E}_{\xi}\left[\ell(A(S);\xi)\right] - \frac{1}{n}\sum_{i=1}^n \ell(A(S), \xi_i)\right| \le \mathcal{O}\left(\gamma\log(n)\log\left(\frac{1}{\delta}\right) + M\sqrt{\frac{\log(1/\delta)}{n}}\right).
\]
\end{lemma}

\emph{RIP-condition-free convergence rate of IHT.} The rate of convergence and parameter estimation error of IHT have been extensively analyzed under RIP (or restricted strong condition number) bounding conditions~\cite{bahmani2013greedy,yuan2014gradient}. The RIP-type conditions, however, are unrealistic in many applications. To remedy this deficiency, sparsity-level relaxation strategy was considered in~\cite{jain2014iterative,yuan2018gradient} with which the high-dimensional estimation consistency of IHT can be established under arbitrary restricted strong condition number. In order to make our analysis more realistic for high-dimensional problems, we choose to work on the following RIP-condition-free convergence rate bound, which is essentially from~\cite{jain2014iterative}, for IHT invoking on the empirical risk $F_S$.
\begin{lemma}[Convergence rate of IHT]\label{lemma:convergence_iht}
Assume that $F_S$ is $L_{3k}$-smooth and $\mu_{3k}$-strongly convex. Consider $\bar k$ such that $k\ge \frac{32 L_{3k}^2}{\mu_{3k}^2}\bar k$. Let $\bar w_{S,k} =\argmin_{\|w\|_0\le \bar k} F_S(w)$. Set $\eta = \frac{2}{3L_{3k}}$. Then for any $\epsilon>0$, IHT outputs $w_{S,k}^{(t)}$ satisfying $F_S(w_{S,k}^{(t)}) \le F_S(\bar w_{S,k}) + \epsilon$, after
\[
t\ge\mathcal{O}\left(\frac{L_{3k}}{\mu_{3k}}\log \left(\frac{F_S(w^{(0)}_{S,k})}{\epsilon}\right)\right)
\]
rounds of iteration.
\end{lemma}

\emph{Localized Rademacher Complexities and data dependent risk bounds.} Let us define $\|\ell(w;\cdot) - \ell(w';\cdot)\|_\infty:=\max_{\xi\in \mathcal{X}}|\ell(w;\xi)- \ell(w';\xi)|$. We further introduce following concept of Localized Rademacher Complexity which plays an important role in deriving the fast rates of convergence:
\[
R_S(r_n; w^*):= \mathbb{E}_{\varepsilon}\left[\sup_{\|\ell(w;\cdot) - \ell(w^*;\cdot)\|_\infty\le r_n}\left|\frac{1}{n}\sum_{i=1}^n\varepsilon_i [\ell(w;\xi_i) - \ell(w^*;\xi_i)]\right|\right],
\]
where $r_n>0$ and $w^*$ are fixed and $\{\varepsilon_i\}$ are i.i.d. Rademacher random variables, i.e., symmetric Bernoulli random variables taking values $+1$ and $-1$ with probability $1/2$ each. $R_S(\delta;w^*)$ can be used as a data dependent complexity measure of the target parametric class around $w^*$ that allows one to estimate the accuracy of approximation of $F(w) - F(w^*)$ by $F_S(w) - F_S(w^*)$ based on the data. The following localized concentration bound is elementary and it can implied immediately by the symmetrization and McDiarmid's inequalities (see, e.g.,~\cite{koltchinskii2006local}).
\begin{lemma}[Data dependent local concentration bound]\label{lemma:concentration_elementary}
For any fixed $w^*$ and all $\delta\in(0,1)$, the following bound holds with probability at least $1-\delta$:
\[
\sup_{\|\ell(w;\cdot) - \ell(w^*;\cdot)\|_\infty\le r_n}\left|F_S(w) - F_S(w^*) - (F(w) - F(w^*))\right|\le 2R_S(r_n; w^*) + 3r_n\sqrt{\frac{2\log(2/\delta)}{n}}.
\]
\end{lemma}
The following elementary lemma (see, e.g.,~\cite[Lemma 14]{yuan2018gradient}) is useful in our analysis.
\begin{lemma}\label{lemma:estimation_error}
Assume that $f$ is $\mu_s$-strongly convex. Then for any $w,w'$ such that $\|w- w'\|_0 \le s $ and $f(w)\le f(w')+\epsilon$ for some $\epsilon\ge 0$, the following bound holds
\[
\|w - w'\| \le \frac{2\sqrt{s}\|\nabla f(w')\|_\infty}{\mu_s}+ \sqrt{\frac{2\epsilon}{\mu_s}},
\]
where $I = \supp(w)$ and $I'=\supp(w')$.
\end{lemma}

\section{Proofs of Main Results}
\label{apdsect:proof_main_results}

\subsection{Proof of Theorem~\ref{thrm:uniform_stability_iht}}
\label{apdsect:proof_uniform_stability_iht}

In this subsection, we present a detailed proof of Theorem~\ref{thrm:uniform_stability_iht}.

\textbf{A key lemma.} For a given index set $J\subseteq [p]$, let us consider the following restrictive estimator over $J$:
\begin{equation}\label{equat:w_SJ}
w_{S\mid J} = \argmin_{w\in \mathcal{W}, \supp(w)\subseteq J} F_S(w).
\end{equation}
We present the following lemma about the uniform generalization gap of $w_{S\mid J}$ for all $J$ with $|J|=k$ which is crucial to our proof.
\begin{lemma}\label{lemma:support_stability}
Assume that the loss function $\ell$ is $G$-Lipschitz continuous with respect to its first argument and $\ell(\cdot;\xi)\le M$ for all $\xi$. Suppose that $F_S$ is $\mu_k$-strongly convex with probability at least $1- \delta'_n$ over the random draw of $S$. Let $\mathcal{J}=\{J \subseteq [p]: |J|= k\}$ be the set of index set of cardinality $k$. Then for any $\delta\in(0,1-\delta'_n)$ and $\lambda>0$, it holds with probability at least $1-\delta-\delta'_n$ over the random draw of $S$ that
\[
\begin{aligned}
&\sup_{J\subseteq \mathcal{J}} \left|F(w_{S\mid J}) - F_S(w_{S \mid J})\right| \\
\le&
\mathcal{O}\left(\frac{G^2}{\lambda n}\log(n)\left(\log\left(\frac{1}{\delta}\right)+k\log\left(\frac{ep}{k}\right)\right)+M\sqrt{\frac{\log(1/\delta)+ k\log(ep/k)}{n}}+ \frac{\lambda G \sqrt{M}}{\mu_k\sqrt{\mu_k}}\right).
\end{aligned}
\]
\end{lemma}
\begin{proof}
Let us consider the following defined $\ell_2$-regularized $\ell_0$-ERM estimator for any given $\lambda>0$:
\[
w_{\lambda,S\mid J} := \argmin_{w\in \mathcal{W}, \supp(w)\subseteq J} \left\{F_{\lambda,S}(w):= F_S(w) + \frac{\lambda}{2}\|w\|^2\right\}.
\]
The reason for introducing the additional $\ell_2$-regularization term is to guarantee uniform stability of the hypothetical estimator $w_{\lambda,S\mid J}$. Based on the standard proof arguments~\cite{shalev2009stochastic} we can show that the optimal model $w_{\lambda,S\mid J}$ has uniform stability $\gamma=\frac{4G^2}{\lambda n}$. Indeed, let $S^{(i)}$ be a sample set that is identical to $S$ except that one of the $\xi_i$ is replaced by another random sample $\xi'_i$. Then we can derive that
\[
\begin{aligned}
&F_{\lambda,S}(w_{\lambda, S^{(i)}\mid J}) - F_{\lambda,S}(w_{\lambda,S\mid J}) \\
=& \frac{1}{n}\sum_{j\neq i} \left(\ell(w_{\lambda, S^{(i)}\mid J};\xi_j) - \ell(w_{\lambda,S\mid J};\xi_j) \right) + \frac{1}{n} \left(\ell(w_{\lambda, S^{(i)}\mid J};\xi_i) - \ell( w_{\lambda,S\mid J};\xi_i)\right)\\
& + \frac{\lambda}{2}\|w_{\lambda, S^{(i)}\mid J}\|^2 -  \frac{\lambda}{2}\|w_{\lambda,S\mid J}\|^2\\
=& F_{\lambda,S^{(i)}}(w_{\lambda, S^{(i)}\mid J}) - F_{\lambda,S^{(i)}}(w_{\lambda,S\mid J}) + \frac{1}{n} \left(\ell(w_{\lambda, S^{(i)}\mid J};\xi_i) - \ell(w_{\lambda,S\mid J};\xi_i) \right) \\
& - \frac{1}{n} \left(\ell(w_{\lambda, S^{(i)}\mid J};\xi'_i) - \ell(w_{\lambda,S\mid J};\xi'_i) \right) \\
\le& \frac{1}{n} \left|\ell(w_{\lambda, S^{(i)}\mid J};\xi_i) - \ell(w_{\lambda,S\mid J};\xi_i) \right| + \frac{1}{n} \left|\ell(w_{\lambda, S^{(i)}\mid J};\xi'_i) - \ell(w_{\lambda,S\mid J};\xi'_i)\right| \\
\le& \frac{2G}{n} \|w_{\lambda, S^{(i)}\mid J} - w_{\lambda,S\mid J}\|,
\end{aligned}
\]
where we have used the optimality of $w_{\lambda, S^{(i)}\mid J}$ with respect to $F_{\lambda,S^{(i)}}(w)$ and the Lipschitz continuity of the loss function $\ell(w;\xi)$. Since $F_{\lambda,S}$ is $\lambda$-strongly convex and $w_{\lambda,S\mid J}$ is optimal for $F_{\lambda,S}(w)$ over the supporting set $J$, we have
\[
F_{\lambda,S}(w_{\lambda, S^{(i)}\mid J}) \ge F_{\lambda,S}(w_{\lambda,S\mid J}) + \frac{\lambda}{2}\|w_{\lambda, S^{(i)}\mid J} - w_{\lambda,S\mid J}\|^2.
\]
By combing the preceding two inequalities we arrive at $\|w_{\lambda, S^{(i)}\mid J} - w_{\lambda,S\mid J}\| \le \frac{4G}{\lambda n}$. Consequently from the Lipschitz continuity of $\ell$ we have that for any sample $\xi$
\[
|\ell(w_{\lambda, S^{(i)}\mid J};\xi ) - \ell(w_{\lambda,S\mid J};\xi) | \le G\|w_{\lambda,S\mid J}^{(i)} - w_{\lambda,S\mid J}\| \le \frac{4G^2}{\lambda n}.
\]
This confirms that the optimal model $w_{\lambda,S\mid J}$ has uniform stability $\gamma=\frac{4G^2}{\lambda n}$. By invoking Lemma~\ref{lemma:uniform_stability} we obtain that with probability at least $1-\delta$ over random draw of $S$,
\begin{equation}\label{inequat:lemma_support_stability_key_1}
\left|F(w_{\lambda, S\mid J}) - F_S(w_{\lambda, S \mid J}) \right| \le \mathcal{O}\left(\frac{G^2}{\lambda n}\log(n)\log\left(\frac{1}{\delta}\right) + M\sqrt{\frac{\log(1/\delta)}{n}}\right).
\end{equation}
Let $\mathcal{J}=\{J \subseteq [p]: |J|= k\}$ be the set of index set of cardinality $k$. It is standard to verify $|\mathcal{J}|=\binom{p}{k}\le \left(\frac{ep}{k}\right)^k$~\cite[Lemma 2.7]{rigollet201518}. Then for each $J\in \mathcal{J}$, based on~\eqref{inequat:lemma_support_stability_key_1} we must have that with probability at least $1-\frac{\delta}{|\mathcal{J}|}$ over $S$, the following generalization gap is valid for any $\lambda>0$:
\[
\left|F(w_{\lambda, S\mid J})- F_S(w_{\lambda, S\mid J})\right| \le \mathcal{O}\left( \frac{G^2}{\lambda n}\log(n)\log\left(\frac{|\mathcal{J}|}{\delta}\right)+M\sqrt{\frac{\log(|\mathcal{J}|/\delta)}{n}}\right).
\]
Then we obtain that with probability at least $1-\delta$, $\sup_{J \subseteq \mathcal{J}} \left|F(w_{\lambda, S\mid J})- F_S(w_{\lambda, S\mid J})\right|$ is upper bounded by
\begin{equation}\label{inequat:lemma_support_stability_key_2}
\mathcal{O}\left(\frac{G^2}{\lambda n}\log(n)\left(\log\left(\frac{1}{\delta}\right)+k\log\left(\frac{ep}{k}\right)\right)+M\sqrt{\frac{\log(1/\delta)+ k\log(ep/k)}{n}}\right).
\end{equation}

Next, we show how to bound the estimator difference $\sup_{J\subseteq \mathcal{J}}\|w_{S\mid J} - w_{\lambda, S\mid J}\|$. The strong convexity assumption of $F_S$ implies that the following bound holds with probability at least $1-\delta'_n$ over $S$ for all $J\subseteq \mathcal{J}$:
\[
\lambda \|w_{S\mid J}\| = \|\nabla_J F_{\lambda, S}(w_{S\mid J}) - \nabla_J F_{\lambda, S}(w_{\lambda, S\mid J})\| \ge (\mu_k+\lambda)\|w_{S\mid J} - w_{\lambda, S\mid J}\|,
\]
where the notation $\nabla_J g $ denotes the restriction of gradient $\nabla g$ over $J$ and we have used the optimality of $w_{\lambda,S\mid J}$ and $w_{S\mid J}$ over $J$ which implies that
\[
\nabla_J F_{\lambda, S}(w_{\lambda, S\mid J})=0, \quad \nabla_J F_{\lambda, S}(w_{S\mid J}) = \nabla_J F_S(w_{S\mid J}) + \lambda w_{S\mid J}=\lambda w_{S\mid J}.
\]
In the meanwhile, since $\ell(0;\cdot)\in (0, M)$, we must have the following bound holds with probability at least $1-\delta'_n$ over $S$ for all $J\subseteq \mathcal{J}$:
\[
M \ge F_S(0) \ge F_S(0) - F_S(w_{S\mid J}) \ge \frac{\mu_{k}}{2}\|w_{S\mid J}\|^2,
\]
which leads to $\|w_{S\mid J}\|\le \sqrt{2 M/\mu_{k}}$. Then it follows readily from the previous two inequalities that
\[
\|w_{S\mid J} - w_{\lambda, S\mid J}\| \le \frac{\lambda}{\mu_k + \lambda}\|w_{S \mid J}\| \le \frac{\lambda \sqrt{2 M}}{\sqrt{\mu_k}(\mu_k + \lambda)}\le \frac{\lambda \sqrt{2 M}}{\mu_k\sqrt{\mu_k}}.
\]
Since the loss function is $G$-Lipschitz continuous, the following is then valid with probability at least $1-\delta'_n$ over the random draw of $S$ for all $J\subseteq \mathcal{J}$:
\[
\begin{aligned}
& \left|F(w_{S\mid J}) - F_S(w_{S \mid J}) \right| \\
\le & \left|F(w_{\lambda, S\mid J}) - F_S(w_{\lambda, S \mid J})\right| + |F_S(w_{S\mid J})-F_S(w_{\lambda, S\mid J})| + |F(w_{S\mid J})-F(w_{\lambda, S\mid J})| \\
\le& \left|F(w_{\lambda, S\mid J}) - F_S(w_{\lambda, S \mid J})\right| + 2G\|w_{S\mid J} - w_{\lambda, S\mid J}\| \\
\le& \left|F(w_{\lambda, S\mid J}) - F_S(w_{\lambda, S \mid J})\right| + \frac{2\lambda G \sqrt{2M}}{\mu_k\sqrt{\mu_k}}.
\end{aligned}
\]
In view of the above bound and the bound in~\eqref{inequat:lemma_support_stability_key_2}, with probability at least $1-\delta-\delta'_n$ over $S$ we have $\sup_{J\subseteq \mathcal{J}} \left|F(w_{S\mid J}) - F_S(w_{S \mid J})\right| \le$
\[
\mathcal{O}\left(\frac{G^2}{\lambda n}\log(n)\left(\log\left(\frac{1}{\delta}\right)+k\log\left(\frac{ep}{k}\right)\right)+M\sqrt{\frac{\log(1/\delta)+ k\log(ep/k)}{n}}+ \frac{\lambda G \sqrt{M}}{\mu_k\sqrt{\mu_k}}\right).
\]
The proof is concluded.
\end{proof}

Now we are in the position to prove Theorem~\ref{thrm:uniform_stability_iht}.

\begin{proof}[Proof of Theorem~\ref{thrm:uniform_stability_iht}]
Let $\mathcal{J}=\{J \subseteq [p]: |J|= k\}$ be the set of index set of cardinality $k$. For any random sample set $S$, by the definition of $\tilde w^{(t)}_{S,k}$ we always have $\tilde w^{(t)}_{S,k} \in \{w_{S\mid J}: J\in \mathcal{J}\}$. Applying Lemma~\ref{lemma:support_stability} yields that with probability at least $1-\delta-\delta'_n$, the generalization gap $\left|F(\tilde w^{(t)}_{S,k})- F_S(\tilde w^{(t)}_{S,k})\right| $ is upper bounded by
\[
\mathcal{O}\left(\frac{G^2}{\lambda n}\log(n)\left(\log\left(\frac{1}{\delta}\right)+k\log\left(\frac{ep}{k}\right)\right)+M\sqrt{\frac{\log(1/\delta)+ k\log(ep/k)}{n}}+ \frac{\lambda G \sqrt{M}}{\mu_k\sqrt{\mu_k}}\right).
\]
Setting $\lambda=\sqrt{\frac{G\mu_k^{1.5}\log(n)(\log(1/\delta)+k\log(ep/k))}{nM^{0.5}}}$ in the above and preserving leading terms yields
\begin{equation}\label{inequat:generalization_gap_bound}
\left|F(\tilde w^{(t)}_{S,k}) - F_S(\tilde w^{(t)}_{S,k})\right| \le \mathcal{O}\left(\frac{G^{3/2}M^{1/4}}{\mu_k^{3/4}}\sqrt{\frac{\log(n)(\log(1/\delta)+k\log(ep/k))}{n}}\right).
\end{equation}
For any $\epsilon>0$, given that $t=\mathcal{O}\left(\frac{L_{3k}}{\mu_{3k}}\log \left(\frac{F_S(w^{(0)}_{S,k})}{\epsilon}\right)\right)=\mathcal{O}\left(\frac{L_{3k}}{\mu_{3k}}\log \left(\frac{M}{\epsilon}\right)\right)$ is sufficiently large, we can bound the sparse excess risk $F(\tilde w^{(t)}_{S,k}) - F(\bar w)$ as
\[
\begin{aligned}
 F(\tilde w^{(t)}_{S,k}) - F(\bar w) =&  F(\tilde w^{(t)}_{S,k}) - F_S(\tilde w^{(t)}_{S,k}) + F_S(\tilde w^{(t)}_{S,k}) -F_S(\bar w) + F_S(\bar w) - F(\bar w)  \\
\le& \left| F(\tilde w^{(t)}_{S,k}) - F_S(\tilde w^{(t)}_{S,k}) \right| + \left|F_S(\bar w) - F(\bar w)\right| + \epsilon,
\end{aligned}
\]
where in the last inequality we have used the bound $F_S(\tilde w^{(t)}_{S,k}) \le F_S(w^{(t)}_{S,k}) \le F_S(\bar w) + \epsilon$ which is implied by the definition of $\tilde w_{S,k}^{(t)}$ and Lemma~\ref{lemma:convergence_iht}. Since $\ell(\bar w;\xi)\le M$, from  Hoeffding's inequality we know that with probability at least $1-\delta/2$,
\[
\left|F_S(\bar w) - F(\bar w)\right| \le \mathcal{O}\left(M\sqrt{\frac{\log(1/\delta)}{n}}\right).
\]
Based on the generalization gap bound~\eqref{inequat:generalization_gap_bound} and by union probability we get with probability at least $1-\delta$
\[
\begin{aligned}
&F(\tilde w^{(t)}_{S,k}) - F(\bar w) \\
\le& \left|F(\tilde w^{(t)}_{S,k}) - F_S(\tilde w^{(t)}_{S,k})\right| + \left|F_S(\bar w) - F(\bar w)\right| + \epsilon \\
\le& \mathcal{O}\left(\frac{G^{3/2}M^{1/4}}{\mu_k^{3/4}}\sqrt{\frac{\log(n)(\log(1/\delta)+k\log(ep/k))}{n}} + M\sqrt{\frac{\log(1/\delta)}{n}} + \epsilon \right).
\end{aligned}
\]
Setting $\epsilon=\mathcal{O}(\sqrt{k\log(n)\log(ep/k)/n})$ yields the desired bound (keep in mind the monotonicity of restricted smoothness and strong convexity). This completes the proof.
\end{proof}

\subsection{Proof of Corollary~\ref{corol:generalization_barw_logisticreg}}
\label{apdsect:proof_uniform_stability_iht_logisticreg}

In this subsection we prove Corollary~\ref{corol:generalization_barw_logisticreg} which is an application of Theorem~\ref{thrm:uniform_stability_iht} to sparse logistic regression models. We first present the following lemma, which follows immediately from~\cite[Lemma 6]{agarwal2012fast}, to be used for proving the main result.
\begin{lemma}\label{lemma:strong_convexity_subgaussian}
Suppose $x_i$ are drawn i.i.d. from a zero-mean sub-Gaussian distribution with covariance matrix $\Sigma\succ 0$. Let $X = [x_1,...,x_n]\in \mathbb{R}^{d\times n}$. Assume that $\Sigma_{jj}\le \sigma^2$. Then there exist universal positive constants $c_0$ and $c_1$ such that for all $w\in \mathbb{R}^p$
\[
\begin{aligned}
\frac{\|X^\top w\|^2}{n} \ge& \frac{1}{2}\|\Sigma^{1/2}w\|^2 - c_1\frac{\sigma^2\log (p)}{n} \|w\|^2_1
\end{aligned}
\]
holds with probability at least $1- \exp\{-c_0 n\}$.
\end{lemma}

\begin{proof}[Proof of Corollary~\ref{corol:generalization_barw_logisticreg}]
Given that $\|x_i\|\le 1$, we have $\ell(w;\xi_i)$ is $L$-smooth with $L\le4s(2y_iw^\top x_i)(1-s(2y_iw^\top x_i))\le1$. Since $\|w\|\le R$, we must have $|y_iw^\top x_i|\le R$ and thus the logistic loss $\ell(w;\xi_i) = \log(1+\exp(-2y_i w^\top x_i))$ satisfies $\ell(w;\xi_i) \le \mathcal{O}(R)$ and $[\Lambda(w)]_{ii}=4s(2y_iw^\top x_i)(1-s(2y_iw^\top x_i))\ge \frac{4}{(1+\exp(2R))^2}\ge\frac{1}{\exp(4R)}$. It follows that
\[
\nabla^2 F_S(w) = \frac{1}{n}X \Lambda(w) X^\top \succeq \frac{1}{n\exp(4R)} XX^\top = \frac{1}{\exp(4R)}\Sigma.
\]
In view of Lemma~\ref{lemma:strong_convexity_subgaussian} and the fact $\|w\|_1\le \sqrt{k}\|w\|$ when $\|w\|_0\le k$ we can verify that with probability at least $1- \exp\{-c_0 n\}$, $F_S(w)$ is $\mu_{4k}$-strongly convex with
\[
\mu_{4k}= \frac{1}{\exp(4R)} \left(\frac{1}{2}\lambda_{\min}(\Sigma) - \frac{k c_1\log (p)}{n}\right).
\]
Provided that $n \ge \frac{4kc_1 \log (p)}{\lambda_{\min}(\Sigma)}$, we have $\mu_{4k}\ge \frac{\lambda_{\min}(\Sigma)}{4\exp(4R)}$ holds with probability at least $1- \exp\{-c_0 n\}$. By invoking Theorem~\ref{thrm:uniform_stability_iht}, after sufficiently large $T\ge\mathcal{O}\left(\frac{\exp(R)}{\lambda_{\min}(\Sigma)}\log \left(\frac{nR}{k\log(n)\log(p/k)}\right)\right)$ rounds of IHT iteration, with probability at least $1 - \delta - \exp\{-c_0 n\}$ the sparse excess risk of IHT converges at the rate of
\[
\mathcal{O}\left(\frac{\exp(R)}{\lambda^{3/4}_{\min}(\Sigma)}\sqrt{\frac{\log(n)(\log(1/\delta)+k\log(p/k))}{n}} + R\sqrt{\frac{\log(1/\delta)}{n}}\right).
\]
This completes the proof.
\end{proof}

\subsection{Proof of Theorem~\ref{thrm:uniform_stability_strong_iht}}
\label{apdsect:proof_uniform_stability_strong_iht}

In this subsection, we present the detailed proof of Theorem~\ref{thrm:uniform_stability_strong_iht}. In what follows, we will frequently use the operator $\mathrm{H}_J(w)$ which is defined as the restriction of $w$ over an index set $J$. We also will use the abbreviation $\mathrm{H}_J(\nabla F(w))=\nabla_J F(w)$ for the sake of notation simplicity. The following lemma is simple yet useful in our analysis.
\begin{lemma}\label{lemma:strong_smooth}
Assume that a differentiable function $f$ is $\mu_s$-strongly convex and $L_s$-smooth. For any index set $J$ with cardinality $|J| \le s$ and any $w,w'$
with $\supp(w)\cup \supp(w')\subseteq J$, if $\eta\in(0, 2 / (L_s + \mu_s) )$, then
\[
\left\|w - w' - \eta \nabla_J f(w) + \eta \nabla_J f(w')\right\| \le \left(1 - \frac{\eta L_s\mu_s}{L_s + \mu_s}\right)\|w - w'\|.
\]
\end{lemma}
\begin{proof}
Since $f$ is $\mu_s$-strongly convex over $J$, we have that $g(w)=f(w) - \frac{\mu_s\|w\|^2}{2}$ is convex and $(L_s-\mu_s)$-smooth when restricted to $J$. Then based on the co-coercivity of $\nabla g$ we know that
\[
\langle \nabla_J g(w) - \nabla_J g(w'), w - w'\rangle \ge \frac{1}{L_s - \mu_s}\| \nabla_J g(w) - \nabla_J g(w')\|^2,
\]
which then yields
\[
\langle \nabla_J f(w) - \nabla_J f(w'), w - w'\rangle \ge \frac{1}{L_s+\mu_s}\|\nabla_J f(w) - \nabla_J f(w')\|^2 + \frac{L_s\mu_s}{L_s+\mu_s}\|w-w'\|^2.
\]
Based on this inequality we can show
\[
\begin{aligned}
&\|w - w' - \eta \nabla_J f(w) + \eta \nabla_J f(w')\|^2 \\
=& \|w-w'\|^2 - 2\eta \langle \nabla_J f(w) - \nabla_J f(w'), w - w'\rangle + \eta^2 \|\nabla_J f(w) - \nabla_J f(w')\|^2\\
\le& \left(1 - \frac{2\eta L_s\mu_s}{L_s+\mu_s}\right)\|w - w'\|^2 - \left(\frac{2\eta}{L_s+\mu_s} - \eta^2\right)\| \nabla_J f(w) - \nabla_J f(w')\|^2 \\
\le& \left(1 - \frac{2\eta L_s\mu_s}{L_s+\mu_s}\right)\|w - w'\|^2 \le \left(1 - \frac{\eta L_s\mu_s}{L_s+\mu_s}\right)^2\|w - w'\|^2,
\end{aligned}
\]
where in the last but one inequality we have used the assumption on $\eta$ and the last inequality follows from $1-2a\le (1-a)^2$. This readily implies the desired bound.
\end{proof}

The following key lemma shows that if the population function $F$ is IHT stable, then the supporting set of the sparse solution returned by IHT invoked on the empirical risk $F_S$ is also unique provided that $F_S$ is close enough to $F$ along the solution path of IHT.
\begin{lemma}\label{lemma:emp_iht_stability}
For a fixed data sample $S$, assume that $F_S$ is $\mu_{4k}$-strongly convex and $L_{4k}$-smooth. Let $\{w^{(t)}\}_{t=1}^T$ and $\{w_{S,k}^{(t)}\}_{t=1}^T$ respectively be the sequence generated by invoking IHT on $F$ and $F_S$ with step-size $\eta= \frac{2}{3L_{4k}}$ and initialization $w^{(0)}$. Suppose that the population risk function $F$ is $(\varepsilon_{k},\eta, T, w^{(0)})$-IHT stable and $\left\|\nabla F_S(w^{(t)}) - \nabla F(w^{(t)})\right\|\le \frac{L_{4k}\mu_{4k}\varepsilon_k}{2(L_{4k}+\mu_{4k})}$, $\forall t\in [T]$. Then we have
\[
\|w^{(t)}_{S,k} - w^{(t)}\| < \varepsilon_k/2, \ \ \supp\left(w^{(t)}_{S,k}\right) = \supp\left(w^{(t)}\right), \  \ \forall t\in [T].
\]
\end{lemma}
\begin{proof}
We show by induction that $\forall t\in \{0\} \cup[T]$, $\|w^{(t)}_{S,k} - w^{(t)}\| < \varepsilon_k/2$ and $\supp(w_{S,k}^{(t)})=\supp(w^{(t)})$. The base case $t=0$ holds trivially as  $w_{S,k}^{(0)}=w^{(0)}$. Suppose that the claim holds for some $t\ge 0$. Now consider the case $t+1$. Denote $J_S^{(\tau)}=\supp(w_{S,k}^{(\tau)})$ and $J^{(\tau)}=\supp(w^{(\tau)})$ for $\tau=t,t+1$ and $J=\bigcup_{\tau=t}^{t+1} \left(J_S^{(\tau)}\bigcup J^{(\tau)}\right)$. Then we must have $|J|\le 4k$. Let us consider the following pair of vectors:
\[
\hat w_{S,k}^{(t+1)}:= \mathrm{H}_J \left(w_{S,k}^{(t)} - \eta \nabla F_S(w_{S,k}^{(t)})\right), \quad\hat w^{(t+1)}:=\mathrm{H}_J \left( w^{(t)} - \eta \nabla F(w^{(t)})\right).
\]
We can show that
\[
\begin{aligned}
&\left\|\hat w_{S,k}^{(t+1)} - \hat w^{(t+1)}\right\| = \left\| w_{S,k}^{(t)} - \eta \nabla_J F_S(w_{S,k}^{(t)}) - w^{(t)} + \eta \nabla_J F(w^{(t)})\right\| \\
=&  \left\| w_{S,k}^{(t)} - w^{(t)} - \eta \nabla_J F_S(w_{S,k}^{(t)}) + \eta \nabla_J F_S(w^{(t)}) - \eta \nabla_J F_S(w^{(t)}) + \eta \nabla_J F(w^{(t)})\right\| \\
\le& \left\| w_{S,k}^{(t)} - w^{(t)} - \eta \nabla_J F_S(w_{S,k}^{(t)}) + \eta \nabla_J F_S(w^{(t)})\right\| + \eta \left\|\nabla F_S(w^{(t)}) - \nabla F(w^{(t)})\right\| \\
\overset{\zeta_1}{\le}& \left(1-\frac{2\mu_{4k}}{3(L_{4k}+\mu_{4k})}\right)\left\|w_{S,k}^{(t)} - w^{(t)}\right\| + \frac{2}{3L}\left\|\nabla F_S(w^{(t)}) - \nabla F(w^{(t)})\right\| \\
\overset{\zeta_2}{<}&\left(1-\frac{2\mu_{4k}}{3(L_{4k} + \mu_{4k})}\right) \frac{\varepsilon_k}{2} + \frac{2\mu_{4k}}{3(L_{4k}+\mu_{4k})} \frac{\varepsilon_k}{2} = \frac{\varepsilon_k}{2},
\end{aligned}
\]
where in ``$\zeta_1$'' we have used Lemma~\ref{lemma:strong_smooth} with $\eta= 2/(3L_{4k})$, and ``$\zeta_2$'' follows from the induction assumption and the bound on $\left\|\nabla F_S(w^{(t)}) - \nabla F(w^{(t)})\right\|$. By definition $J^{(t+1)}\subseteq J$, and thus it holds trivially that $J^{(t+1)}$ also uniquely contains the top $k$ (in magnitude) entries of $\hat w^{(t+1)}$ as a restriction of $w^{(t+1)}$ over $J$. Based on this observation, since $F$ is $(\varepsilon_{k},\eta, T, w^{(0)})$-IHT stable, $w^{(t)} - \eta \nabla F(w^{(t)})$ must be $\varepsilon_{k}$-hard-thresholding stable which then implies that $\hat w^{(t+1)}$ is also $\varepsilon_{k}$-hard-thresholding stable. Therefore, the preceding inequality readily indicates that $\hat w_{S,k}^{(t+1)}$ and $\hat w^{(t+1)}$ share the identical top $k$ entries, and thus $J_S^{(t+1)}=J^{(t+1)}$. Consequently, based on the preceding inequality we can show that
\[
\left\|w_{S,k}^{(t+1)} - w^{(t+1)}\right\| \le \left\|\hat w_{S,k}^{(t+1)} - \hat w^{(t+1)}\right\| < \varepsilon_k/2.
\]
This shows that the claim holds for $t+1$ and the proof is concluded.
\end{proof}
Now we are ready to prove the main result.
\begin{proof}[Proof of Theorem~\ref{thrm:uniform_stability_strong_iht}]
Let us define an oracle sequence $\{w^{(t)}\}_{t=1}^T$ generated by applying $T$ rounds of IHT iteration to the population risk $F$ with the considered fixed initialization $w^{(0)}$ and step-size $\eta$. Since $F$ is $(\varepsilon_{k},\eta, T, w^{(0)})$-IHT stable, the sequence $\{w^{(t)}\}_{t=1}^T$ is unique and deterministic. Since $\|\nabla \ell(w;\cdot)\|\le G$, from the Hoeffding's concentration bound and union probability we know that with probability at least $1-\frac{\delta}{2}$ over $S$,
\[
\sup_{t\in [T]}\|\nabla F_S(w^{(t)}) - \nabla F(w^{(t)})\| \le G \sqrt{\frac{\log(2pT/\delta)}{2n}} \le \frac{L_{4k}\mu_{4k}\varepsilon_k}{2(L_{4k}+\mu_{4k})},
\]
where the last inequality is due to the condition on sample size $n$. The assumptions in the theorem imply that $F_S$ is $L_{4k}$-smooth and $\mu_{4k}$-strongly convex with probability at least $1-\delta'_n$ over $S$. Therefore, by invoking Lemma~\ref{lemma:emp_iht_stability} and union probability we know that with probability at least $1-\delta'_n -\frac{\delta}{2}$ over $S$, the following event occurs:
\[
 \mathcal{E}_1:\left\{ \supp\left(w^{(T)}_{S,k}\right) = \supp\left(w^{(T)}\right) \right\}.
\]
Since $J:=\supp\left(w^{(T)}\right)$ is a fixed deterministic index set of size $k$, using a similar proof argument to that of Lemma~\ref{lemma:support_stability} (keep in mind that $\mu_{4k}\le \mu_k$) we can show the following event occurs holds with probability at least $1-\delta_n'-\frac{\delta}{2}$:
\[
 \mathcal{E}_2:\left\{ F(w_{S\mid J}) - F_S(w_{S\mid J}) \le \mathcal{O}\left(\frac{G^2}{\lambda n}\log(n)\log\left(\frac{1}{\delta}\right) + M\sqrt{\frac{\log(1/\delta)}{n}} + \frac{\lambda G \sqrt{M}}{\mu_{4k}\sqrt{\mu_{4k}}}\right)\right\}.
\]
Now let us consider the event
\[
 \mathcal{E}:\left\{ F(\tilde w^{(T)}_{S,k}) - F_S(\tilde w^{(T)}_{S,k}) \le \mathcal{O}\left(\frac{G^2}{\lambda n}\log(n)\log\left(\frac{1}{\delta}\right) + M\sqrt{\frac{\log(1/\delta)}{n}} + \frac{\lambda G \sqrt{M}}{\mu_{4k}\sqrt{\mu_{4k}}}\right) \right\}.
\]
Since $\mathcal{E}\supseteq \mathcal{E}_1 \cap \mathcal{E}_2$, we must have
\[
\mathbb{P}\left(\mathcal{E} \right) \ge \mathbb{P}\left(\mathcal{E}_1\cap \mathcal{E}_2 \right) \ge 1- \mathbb{P}\left(\overline{\mathcal{E}}_1 \right) -\mathbb{P}\left(\overline{\mathcal{E}}_2 \right) \ge 1 - 2\delta'_n - \delta.
\]
Setting $\lambda=\sqrt{\frac{G\mu_{4k}^{1.5}\log(n)\log(1/\delta)}{nM^{0.5}}}$ in the event $\mathcal{E}$ and preserving the leading terms yields that with probability at least $1-2\delta'_n-\delta$:
\begin{equation}\label{inequat:generalization_gap_bound_strong}
F(\tilde w^{(T)}_{S,k}) - F_S(\tilde w^{(T)}_{S,k}) \le \mathcal{O}\left(\frac{G^{3/2}M^{1/4}}{\mu_{4k}^{3/4}}\sqrt{\frac{\log(n)\log(1/\delta)}{n}}\right).
\end{equation}
Following a similar argument to that of Theorem~\ref{thrm:uniform_stability_iht} we can show that if $T=\mathcal{O}\left(\frac{L_{4k}}{\mu_{4k}}\log \left(\frac{M}{\epsilon}\right)\right)$ is sufficiently large, then with probability at least $1-\delta-2\delta'_n$
\[
F(\tilde w^{(T)}_{S,k}) - F(\bar w) \le \mathcal{O}\left(\frac{G^{3/2}M^{1/4}}{\mu_k^{3/4}}\sqrt{\frac{\log(n)\log(1/\delta)}{n}} + M\sqrt{\frac{\log(1/\delta)}{n}} + \epsilon \right).
\]
Setting $\epsilon=\mathcal{O}(\sqrt{\log(n)\log(1/\delta)/n})$ we obtain the desired bound. This completes the proof.
\end{proof}

\subsection{Proof of Theorem~\ref{thrm:generalizaion_fast_rate}}
\label{apdsect:proof_theorem_generalizaion_fast_rate}

In this subsection, we prove Theorem~\ref{thrm:generalizaion_fast_rate}. For any fixed $J\subseteq[p]$ with $|J|=k$, let
\[
w^*_J = \argmin_{\supp(w) \subseteq J} F(w).
\]
Before proving the main result, we first establish a key lemma which shows a uniform fast rate of $w_{S\mid J}$ (recall the definition in~\eqref{equat:w_SJ}) towards $w^*_J$ for all $J$ if the population risk is restricted strongly convex and the loss is Lipschitz continuous. To ease notation, we define an abbreviation of loss function as $\ell_w(\cdot):= \ell(w;\cdot)$. Particularly, we write the localized Rademacher complexity restricted over $J$ at $w^*_J$ as:
\[
R_{S\mid J}(r_n; w^*_J):= \mathbb{E}_{\varepsilon}\left[\sup_{\supp(w)\subseteq J,\|\ell_w - \ell_{w^*_J}\|_\infty\le r_n}\left|\frac{1}{n}\sum_{i=1}^n\varepsilon_i [\ell_w(\xi_i) - \ell_{w^*_J}(\xi_i)]\right|\right],
\]
where $\{\varepsilon_i\}$ are i.i.d. Rademacher random variables, i.e., symmetric Bernoulli random variables taking values $+1$ and $-1$ with probability $1/2$ each. The following preliminary result is standard yet useful in our analysis. We provide its proof for the sake of completeness.
\begin{lemma}\label{lemma:rademacher_bound}
Under Assumption~\ref{assump:strong_convex_pop}, there exists some absolute constant $C>0$ such that
\[
R_{S\mid J}(Gr_n; w^*_J) \le C G r_n \sqrt{\frac{k}{n}}\log^{3/2}\left(\frac{1}{r_n} \sqrt{\frac{n}{k}}\right).
\]
\end{lemma}
\begin{proof}
Let us restrict the analysis over $\mathcal{W}_J$ as a restriction of $\mathcal{W}$ over $J$. Since $\mathcal{W}$ is assumed to be a subset of unit $\ell_2$-sphere, it is standard (see, for instance,~\cite{boroczky2003covering}) to bound the covering number of $\mathcal{W}_J$ at scale $\epsilon$ with respect to the $\ell_2$-distance as $\log\mathcal{N}(\epsilon, \mathcal{W}_J, \ell_2) \le \mathcal{O}\left(k \log\left(1/\epsilon\right)\right)$. Since the loss function $\ell(w;\xi)$ is $G$-Lipschitz continuous with respect to $w$, it can be verified that the covering number of the class of functions $\mathcal{L}_J=\left\{\xi \mapsto \ell_w(\xi)\mid w \in \mathcal{W}_J\right\}$ with respect to $\ell_\infty$-distance $\|\ell_{w_1} - \ell_{w_2}\|_\infty$ is given by
\[
\log \mathcal{N}(\epsilon, \mathcal{L}_J, \ell_\infty) \le \log \mathcal{N}(\epsilon/G, \mathcal{W}_J, \ell_2) \le \mathcal{O}\left(k\log(G/\epsilon)\right).
\]
Based on the result from~\cite[Lemma A.3]{srebro2010smoothness} on the connection between Rademacher complexity and covering number we can show that
\[
\begin{aligned}
&R_{S\mid J}(Gr_n; w^*_J) \\
\le& \inf_{\alpha>0} \left\{4\alpha + 10 \int_{\alpha}^{Gr_n} \sqrt{\frac{\log \mathcal{N}(\epsilon, \mathcal{L}_J, \ell_\infty)}{n}} d\epsilon \right\} \\
\le& \mathcal{O}\left(4Gr_n\sqrt{\frac{k}{n}} + 10 \int_{Gr_n\sqrt{\frac{k}{n}}}^{G r_n} \sqrt{\frac{k\log(G/\epsilon)}{n}} d\epsilon\right) \\
\le& \mathcal{O}\left(4Gr_n \sqrt{\frac{k}{n}} + 10 G r_n \sqrt{\frac{k}{n}}\int_{Gr_n\sqrt{\frac{k}{n}}}^{G r_n} \frac{\sqrt{\log(G/\epsilon)}}{\epsilon} d\epsilon\right) \\
\overset{\zeta_1}{\le}& \mathcal{O}\left(4Gr_n \sqrt{\frac{k}{n}} + 6.67 G r_n \sqrt{\frac{k}{n}}\log^{3/2}\left(\frac{\sqrt{n}}{r_n \sqrt{k}}\right)\right) \\
\le& \mathcal{O}\left(G r_n \sqrt{\frac{k}{n}}\log^{3/2}\left(\frac{\sqrt{n}}{r_n \sqrt{k}}\right)\right),
\end{aligned}
\]
where in ``$\zeta_1$'' we have used the following fact for $c>b>a>0$:
\[
\int_a^b x^{-1}\sqrt{\log\left(\frac{c}{x}\right)} dx = \frac{2}{3}\left(\log^{3/2}\left(\frac{c}{a}\right) - \log^{3/2}\left(\frac{c}{b}\right)\right) \le \frac{2}{3}\log^{3/2}\left(\frac{c}{a}\right).
\]
This proves the desired bound.
\end{proof}
The following lemma presents a uniform fast rate of $w_{S\mid J}$ for all $J$.
\begin{lemma}\label{lemma:fast_key_1}
Suppose that Assumptions~~\ref{assump:lipschitz},~\ref{assump:strong_convex_pop} are valid. Then for any $\delta\in(0,1)$, it holds with probability at least $1-\delta$ that
\[
\sup_{J\subseteq[p], |J|=k} F(w_{S\mid J}) - F(w^*_J) \le \mathcal{O}\left(\frac{G^2k(\log^{3}(\rho_k n)+ \log(ep/k)) + \log(1/\delta)}{\rho_k n} \right).
\]
\end{lemma}
\begin{proof}
Fix a subset $J$ with $|J|=k$. Let $r_n>0$ be an arbitrary scalar that satisfies
\begin{equation}\label{inequat:r_n_condition}
\frac{\rho_k r_n^2}{2} \ge 2 R_{S\mid J}(G r_n; w^*_J) + \frac{3Gr_n\sqrt{2\log(4/\delta)}}{\sqrt{n}}.
\end{equation}
Our first step is to show that
\begin{equation}\label{inequat:restricted_fast_bound}
\mathbb{P}\left(F(w_{S\mid J}) - F(w^*_J) \le \frac{\rho_k r_n^2}{2} \right) \ge 1 - \delta.
\end{equation}
To this end, suppose the event $\|w_{S\mid J} - w^*_J\| > r_n$ occurs. We can verify that the following event occurs consequently:
\[
\sup_{\supp(w)\subseteq J, \|\ell_w- \ell_{w^*_J}\|_\infty \le Gr_n}\left| F_S(w) - F_S(w^*_J) - (F(w) - F(w^*_J))\right|\ge 2R_S(Gr_n; w^*_J) + \frac{3Gr_n\sqrt{2\log(4/\delta)}}{\sqrt{n}}.
\]
Indeed, let us consider
\[
\tilde w_J = (1-\eta_n) w^*_J + \eta_n  w_{S\mid J},
\]
where $\eta_n =  \frac{r_n}{\|w_{S\mid J} - w^*_J\|}< 1$. It is direct to verify that $\|\tilde w_J - w^*_J\| = r_n$. Since $F_S$ is convex, we must have
\[
F_S(\tilde w_J) \le  (1-\eta_n) F_S(w^*_J) + \eta_n F_S(w_{S\mid J})\le F_S(w^*_J).
\]
Note that $\|\ell_{\tilde w_J} - \ell_{w^*_J}\|_\infty\le G\|\tilde w_J - w^*_J\|=Gr_n$. Therefore, we have
\[
\begin{aligned}
&\sup_{\supp(w)\subseteq J, \|\ell_w- \ell_{w^*_J}\|\le Gr_n}\left|F_S(w) - F_S(w^*_J) - (F(w) - F(w^*_J))\right| \\
&\ge \left|F_S(\tilde w_J) - F_S(w^*_J) - (F(\tilde w_J) - F(w^*_J))\right| \\
& \ge  \left|F(\tilde w_J) - F(w^*_J)\right|\\
&\overset{\zeta_1}{\ge} \frac{\rho_k}{2}\|\tilde w_J - w^*_J\|^2 = \frac{\rho_k r_n^2}{2} \ge 2R_{S\mid J}(Gr_n; w^*_J) + \frac{3Gr_n\sqrt{2\log(4/\delta)}}{\sqrt{n}},
\end{aligned}
\]
where in ``$\zeta_1$'' we have used Assumption~\ref{assump:strong_convex_pop} and the last inequality follows from~\eqref{inequat:r_n_condition}. Then, invoking Lemma~\ref{lemma:concentration_elementary} over the supporting set $J$ yields
\[
\begin{aligned}
& \mathbb{P}\left(\|w_{S\mid J} - w^*_J\| > r_n \right) \\
\le& \mathbb{P} \left(\sup_{\supp(w)\subseteq J, \|\ell_w- \ell_{w^*_J}\|\le Gr_n}\left|F_S(w) - F_S(w^*_J) - (F(w) - F(w^*_J))\right|\ge 2R_{S\mid J}(Gr_n; w^*_J) + \frac{3Gr_n\sqrt{2\log(4/\delta)}}{\sqrt{n}}\right)\\
\le& \frac{\delta}{2}.
\end{aligned}
\]
Now let us consider the following three events:
\[
\begin{aligned}
&\mathcal{E}_1:\left\{ F(w_{S\mid J}) - F(w^*_J) \le 2R_{S\mid J}(Gr_n; w^*_J) + \frac{3Gr_n\sqrt{2\log(4/\delta)}}{\sqrt{n}} \right\} ,\\
&\mathcal{E}_2: \left\{\|w_{S\mid J} - w^*_J\| \le r_n\right\} , \\
&\mathcal{E}_3:\left\{\sup_{\supp(w)\subseteq J, \|\ell_w- \ell_{w^*_J}\|\le Gr_n}\left|F_S(w) - F_S(w^*_J) - (F(w) - F(w^*_J))\right|\le 2R_{S\mid J}(Gr_n; w^*_J) + \frac{3Gr_n\sqrt{2\log(4/\delta)}}{\sqrt{n}} \right\}.
\end{aligned}
\]
Note that
\[
\begin{aligned}
&\|w_{S\mid J} - w^*_J\| \le r_n \\
\Rightarrow & \|\ell_{w_{S\mid J}}- \ell_{w^*_J}\|\le Gr_n \\
\Rightarrow &  F(w_{S\mid J}) - F(w^*_J) \le \sup_{\supp(w)\subseteq J, \|\ell_w- \ell_{w^*_J}\|\le Gr_n}\left|F_S(w) - F_S(w^*_J) - (F(w) - F(w^*_J))\right|.
\end{aligned}
\]
Therefore, we must have
\[
\mathbb{P}\left(\mathcal{E}_1\right) \ge \mathbb{P}\left(\mathcal{E}_1\cap \mathcal{E}_2\right) \ge \mathbb{P}\left(\mathcal{E}_3\cap \mathcal{E}_2\right) \ge 1 -  \mathbb{P}\left(\overline{\mathcal{E}_2}\right) - \mathbb{P}\left(\overline{\mathcal{E}_3}\right) \ge 1 - \frac{\delta}{2} - \frac{\delta}{2} = 1 - \delta,
\]
which together with~\eqref{inequat:r_n_condition} implies the desired bound in~\eqref{inequat:restricted_fast_bound}.

The next step is to properly choose $r_n$ so as to fulfill the key condition of~\eqref{inequat:r_n_condition}. Based on the bound on $R_{S\mid J}(Gr_n; w^*_J)$ as summarized in Lemma~\ref{lemma:rademacher_bound}, there exists some $C>0$ such that
\[
\begin{aligned}
&\frac{\rho_k r_n^2}{2} \ge 2 R_{S\mid J}(G r_n; w^*_J) + \frac{3Gr_n\sqrt{2\log(4/\delta)}}{\sqrt{n}}\\
\Leftarrow&
\frac{\rho_k r_n^2}{2} \ge 2 C G r_n \sqrt{\frac{k}{n}}\log^{3/2}\left(\frac{\sqrt{n}}{r_n \sqrt{k}}\right) + \frac{3Gr_n\sqrt{2\log(4/\delta)}}{\sqrt{n}} \\
\Leftarrow&  r_n \ge \frac{4CG\sqrt{k}\log^{3/2}\left(\frac{\sqrt{n}}{r_n \sqrt{k}}\right)+6G\sqrt{2\log(4/\delta)}}{\rho_k \sqrt{n}}.
\end{aligned}
\]
Therefore, it suffices to choose
\[
r_n = \mathcal{O}\left(\frac{G\sqrt{k}\log^{3/2}\left(\rho_k n\right)+G\sqrt{\log(1/\delta)}}{\rho_k\sqrt{n}}\right) \le \mathcal{O}\left(\frac{G}{\rho_k}\sqrt{\frac{k\log^3(\rho_k n)+\log(1/\delta)}{n}}\right).
\]
Substituting the above choice of $r_n$ to~\eqref{inequat:restricted_fast_bound} yields
\[
F(w_{S\mid J}) - F(w^*_J) \le \mathcal{O}\left(\frac{G^2k\log^{3}(\rho_k n) + \log(1/\delta)}{\rho_k n} \right).
\]
As the final step, since there are at most $\binom{p}{k}\le \left(\frac{ep}{k}\right)^k$ different $J$, by union probability we get
\[
\sup_{J\subseteq[p], |J|=k} F(w_{S\mid J}) - F(w^*_J) \le \mathcal{O}\left(\frac{G^2k(\log^{3}(\rho_k n)+ \log(ep/k)) + \log(1/\delta)}{\rho_k n} \right).
\]
This completes the proof.
\end{proof}

To prove the main result, we also need to prove the following lemma which basically provides a sufficient condition to guarantee the support recovery performance of IHT.
\begin{lemma}\label{lemma:support_recovery_bar_w}
Suppose that $F_S$ is $\mu_{2k}$-strongly convex with probability at least $1- \delta'_n$. Assume that the loss function $\ell$ is $G$-Lipschitz. Suppose that there exists a $\bar k$-sparse vector $\bar w$ such that
\[
\bar w_{\min}> \frac{2\sqrt{2k}\|\nabla F(\bar w)\|_\infty}{\mu_{2k}} + \frac{3G}{\mu_{2k}} \sqrt{\frac{k\log(p/\delta)}{n}}
\]
for some $\delta\in (0, 1-\delta'_n)$. Then for sufficiently large $T\ge\mathcal{O}\left(\frac{L_{2k}}{\mu_{2k}}\log \left(\frac{n\mu_{2k}}{kG\log(n)\log(p/k)}\right)\right)$ rounds of IHT iteration, the support recovery $\supp(\bar w) \subseteq \supp(w_{S,k}^{(T)})$ holds with probability at least $1-\delta - \delta'_n$.
\end{lemma}
\begin{proof}
Let us consider a fixed $\bar w$. Since the $G$-Lipschitz condition implies $\|\nabla \ell(\bar w;\cdot)\|\le G$, from the Hoeffding concentration bound we know that with probability at least $1-\delta$ over $S$,
\[
\|\nabla F_S(\bar w) - \nabla F(\bar w)\| \le G \sqrt{\frac{\log(p/\delta)}{2n}}.
\]
Then with probability at least $1-\delta$,
\begin{equation}\label{inequat:proof_thrm_fast_key1}
\begin{aligned}
&\|\nabla F_S(\bar w)\|_\infty \le \|\nabla F(\bar w)\|_\infty + \|\nabla F_S(\bar w) - \nabla F(\bar w)\|_\infty \\
\le& \|\nabla F(\bar w)\|_\infty + \|\nabla F_S(\bar w) - \nabla F(\bar w)\|\le \|\nabla F(\bar w)\|_\infty + G \sqrt{\frac{\log(p/\delta)}{2n}}.
\end{aligned}
\end{equation}
Since with probability at least $1- \delta'_n$ the empirical risk $F_S$ is $\mu_{2k}$-strongly convex, the bound in Lemma~\ref{lemma:convergence_iht} implies that the following holds for sufficiently large $T\ge\mathcal{O}\left(\frac{L_{2k}}{\mu_{2k}}\log \left(\frac{n\mu_{2k}}{kG\log(n)\log(p/k)}\right)\right)$ with probability at least $1- \delta'_n$:
\[
F_S(\tilde w^{(T)}_{S,k}) \le F_S(\bar w) + \frac{G^2 k\log(p/\delta)}{2\mu_{2k}n}.
\]
Invoking Lemma~\ref{lemma:estimation_error} to the above with $w = \tilde w^{(T)}_{S,k}, w'=\bar w$ and $\epsilon = \frac{kG^2\log(1/\delta)}{2\mu_{2k}n}$ yields that with probability at least $1- \delta'_n$,
\[
\|\tilde w^{(T)}_{S,k} - \bar w\| \le \frac{2\sqrt{2k}\|\nabla F_S(\bar w)\|_\infty}{\mu_{2k}}+ \sqrt{\frac{2\epsilon}{\mu_{2k}}} = \frac{2\sqrt{2k}\|\nabla F_S(\bar w)\|_\infty}{\mu_{2k}}+ \frac{G}{\mu_{2k}}\sqrt{\frac{k\log(p/\delta)}{n}}.
\]
Using~\eqref{inequat:proof_thrm_fast_key1} and union probability argument we obtain that with probability at least $1-\delta-\delta'_n$,
\[
\|\tilde w^{(T)}_{S,k} - \bar w\| \le \frac{2\sqrt{2k}\|\nabla F(\bar w)\|_\infty}{\mu_{2k}}+ \frac{3G}{\mu_{2k}}\sqrt{\frac{k\log(p/\delta)}{n}}.
\]
Consequently from the condition on $\bar w_{\min}$ we must have $\supp(\tilde w^{(T)}_{S,k}) \supseteq \supp(\bar w)$ holds with probability at least $1-\delta-\delta'_n$.
\end{proof}

\newpage

We are now ready to prove the main result of Theorem~\ref{thrm:generalizaion_fast_rate}.
\begin{proof}[Proof of Theorem~\ref{thrm:generalizaion_fast_rate}]
In what follows, we denote $\tilde J=\supp(\tilde w^{(T)}_{S,k})$ and $w^*_{\tilde J} = \argmin_{w\in \mathcal{W}, \supp(w)\subseteq \tilde J} F(w)$. Let us define the following three events associated with the sample set $S$:
\[
\begin{aligned}
&\mathcal{E}_1:\left\{F(\tilde w^{(T)}_{S,k})  - F(\bar w) \le \mathcal{O}\left(\frac{G^2k(\log^{3}(\rho n)+ \log(ep/k)) + \log(1/\delta)}{\rho n} \right) \right\}, \\
&\mathcal{E}_2:\left\{F(\tilde w^{(T)}_{S,k})  - F(w^*_{\tilde J})\le \mathcal{O}\left(\frac{G^2k(\log^{3}(\rho n)+ \log(ep/k)) + \log(1/\delta)}{\rho n} \right) \right\}, \\
&\mathcal{E}_3:=\left\{\supp(\bar w) \subseteq \tilde J \right\}.
\end{aligned}
\]
We claim that $\mathcal{E}_1 \cap \mathcal{E}_3 \supseteq \mathcal{E}_2 \cap \mathcal{E}_3$. Indeed, for any $S\in \mathcal{E}_2 \cap \mathcal{E}_3$, we have
\[
\begin{aligned}
&\supp(\bar w) \subseteq \tilde J \\
\Rightarrow& F(\tilde w^{(T)}_{S,k})  - F(\bar w) \le F(\tilde w^{(T)}_{S,k})  - F(w^*_{\tilde J})\le \mathcal{O}\left(\frac{G^2k(\log^{3}(\rho n)+ \log(ep/k)) + \log(1/\delta)}{\rho n} \right),
\end{aligned}
\]
which implies $S\in \mathcal{E}_1$ and thus $S\in \mathcal{E}_1 \cap \mathcal{E}_3$.

Given the condition on $\bar w$ and $\delta'_n \le \frac{\delta}{4}$, it follows from Lemma~\ref{lemma:support_recovery_bar_w} that $\supp(\bar w) \subseteq \tilde J $ holds with probability at least $1-\frac{\delta}{2}$, i.e.,
\[
\mathbb{P}\left(\mathcal{E}_3 \right) \ge 1 - \frac{\delta}{2}.
\]
In the meanwhile, noting $\tilde w^{(T)}_{S,k}=w_{S\mid \tilde J}$ and invoking Lemma~\ref{lemma:fast_key_1} yields
\[
\mathbb{P}\left(\mathcal{E}_2 \right) \ge 1 - \frac{\delta}{2}.
\]
Combining the above leads to
\[
\mathbb{P}(\mathcal{E}_1) \ge\mathbb{P}(\mathcal{E}_1 \cap \mathcal{E}_3) \ge \mathbb{P}(\mathcal{E}_2 \cap \mathcal{E}_3) \ge 1 - \mathbb{P}(\overline{\mathcal{E}}_2) - \mathbb{P}(\overline{\mathcal{E}}_3) \ge 1 - \delta.
\]
This proves the desired bound (keep in mind the monotonicity of restricted smoothness and strong convexity).
\end{proof}

\subsection{Proof of Theorem~\ref{thrm:generalization_barw_iht}}
\label{apdsect:proof_white_box}

We need the following lemma which can be derived based on the concentration bound of sub-Gaussian random variables.

\begin{lemma}\label{lemma:sub_Gaussian_bounds}
Under Assumption~\ref{assump:gradient_sub_gaussian}, for any $\delta\in(0,1)$ it holds with probability at least $1-\delta$ that
\[
\|\nabla F_S(\bar w)\|_\infty \le \sigma\sqrt{\frac{2\log(p/\delta)}{n}}.
\]
\end{lemma}
\begin{proof}
Consider a fixed index $j \in [p]$. Since $\nabla_j \ell(\bar w;\xi)$ are assumed to be $\sigma^2$-sub-Gaussian and $\nabla F(\bar w) = \mathbb{E}_{\xi}\left[\nabla \ell(\bar w;\xi)\right] = 0$, we must have $\nabla_j \ell(\bar w;\xi)$ are zero-mean $\sigma^2$-sub-Gaussian. Thus by Hoeffding inequality we have that for any $\varepsilon >0 $,
\[
\mathbb{P}\left(\left|\nabla_j F_S(\bar w)\right| > \varepsilon \right) = \mathbb{P}\left(\left|\frac{1}{n}\sum_{\xi_i\in S}\nabla_j \ell(\bar w; \xi_i)\right| > \varepsilon \right)\le \exp\left\{-\frac{n\varepsilon^2}{2\sigma^2}\right\}.
\]
By the union  bound we have
\[
\mathbb{P}(\|\nabla F_S(\bar w)\|_\infty > \varepsilon) \le
p\exp\left\{-\frac{n\varepsilon^2}{2\sigma^2}\right\} \nonumber.
\]
By choosing $\varepsilon = \sqrt{\frac{2\sigma^2\log(p/\delta)}{n}}$ in the above inequality we obtain that
with probability at least $1-\delta$,
\[
\|\nabla F_S(\bar w)\|_\infty \le \sqrt{\frac{2\sigma^2\log(p/\delta)}{n}}.
\]
This completes the proof.
\end{proof}

We are now ready to prove the main result of Theorem~\ref{thrm:generalization_barw_iht}.
\begin{proof}[Proof of Theorem~\ref{thrm:generalization_barw_iht}]
Since by assumption $F_S(w)$ is $L_{4k}$-smooth and $\mu_{4k}$-strongly convex with probability at least $1- \delta'_n$, Lemma~\ref{lemma:convergence_iht} shows that $F_S(w^{(T)}_{S,k}) - F_S(\bar w) \le \epsilon$ with probability at least $1- \delta'_n$ provided that $t\ge\mathcal{O}\left(\frac{L_{4k}}{\mu_{4k}}\log\left(\frac{1}{\epsilon}\right)\right)$. Then by invoking Lemma~\ref{lemma:estimation_error} we obtain that with probability at least $1- \delta'_n$,
\[
\|w^{(T)}_{S,k} - \bar w\|^2 \le \frac{16k\|\nabla F_S(\bar w)\|^2_\infty}{\mu_{2k}^2} + \frac{4\epsilon}{\mu_{2k}}\le \frac{16k\|\nabla F_S(\bar w)\|^2_\infty}{\mu_{4k}^2} + \frac{4\epsilon}{\mu_{4k}}.
\]
From Lemma~\ref{lemma:sub_Gaussian_bounds} we know that with probability at least $1-\delta$,
\[
\|\nabla F_S(\bar w)\|_\infty \le \sigma\sqrt{\frac{2\log(p/\delta)}{n}}.
\]
Then by union probability the following holds with probability at least $1-\delta - \delta'_n$:
\[
\|w^{(T)}_{S,k} - \bar w\|^2 \le \frac{32}{\mu_{4k}^2}\left(\frac{k\sigma^2\log (p/\delta)}{n}\right) + \frac{4\epsilon}{\mu_{4k}}.
\]
Based on the Lipschitz smoothness of $F$ we can show
\[
\begin{aligned}
F(w^{(T)}_{S,k}) - F(\bar w) \le \frac{L}{2}\|w^{(T)}_{S,k} - \bar w\|^2 \le \frac{16L}{\mu_{4k}^2}\left(\frac{k\sigma^2\log (p/\delta)}{n}\right) + \frac{2L\epsilon}{\mu_{4k}}.
\end{aligned}
\]
Setting $\epsilon = \frac{1}{\mu_{4k}}\left(\frac{k\sigma^2\log (p/\delta)}{n}\right)$ yields the desired high probability bound of sparse excess risk.
\end{proof}

\subsection{Proofs of Corollary~\ref{corol:generalization_barw_linearreg} and Corollary~\ref{corol:generalization_barw_logisticreg_fast}}
\label{apdsect:proof_corol_generalization_barw_linearreg}

We first prove Corollary~\ref{corol:generalization_barw_linearreg} which is an application of Theorem~\ref{thrm:generalization_barw_iht} to sparse linear regression models.
\begin{proof}[Proof of Corollary~\ref{corol:generalization_barw_linearreg}]
Let $\xi=\{x,\varepsilon\}$ in which $x$ is zero-mean sub-Gaussian with covariance matrix $\Sigma \succ 0$ and $\varepsilon$ is zero-mean $\sigma^2$-sub-Gaussian. Since $x$ and $\varepsilon$ are independent, it can be directly verified that $\nabla F(\bar w)=\mathbb{E}_{\xi}\left[\nabla \ell(\bar w;\xi)\right]=\mathbb{E}_{\varepsilon,x}\left[-\varepsilon x\right]=0$. Given that $\Sigma_{jj} \le 1$, it can be shown that $\nabla_j \ell(\bar w;\xi_i) = -\varepsilon_i [x_i]_j$ are zero-mean $\sigma^2$-sub-Gaussian variables, which indicates that Assumption~\ref{assump:gradient_sub_gaussian} holds. Clearly, $F$ is $L$-smooth with $L=\lambda_{\max}(\Sigma)$.

Based on Lemma~\ref{lemma:strong_convexity_subgaussian} and the fact $\|w\|_1\le \sqrt{k}\|w\|$ when $\|w\|_0\le k$, it holds with probability at least $1- \exp\{-c_0 n\}$ that $F_S(w)$ is $\mu_{4k}$-strongly convex with
\[
\mu_{4k}= \frac{1}{2}\lambda_{\min}(\Sigma) - \frac{k c_1\log (p)}{n}.
\]
Provided that $n \ge \frac{4kc_1 \log (p)}{\lambda_{\min}(\Sigma)}$, we have $\mu_{4k}\ge \frac{1}{4}\lambda_{\min}(\Sigma)$ holds with probability at least $1- \exp\{-c_0 n\}$. Similarly, we can show that $F_S(w)$ is $L_{4k}$-smooth with $L_{4k}=\mathcal{O}(\lambda_{\max}(\Sigma))$. Provided that
\[
T\ge\mathcal{O}\left(\frac{\lambda_{\max}(\Sigma)}{\lambda_{\min}(\Sigma)}\log \left(\frac{n \lambda_{\min}(\Sigma)}{k\sigma^2\log(p/\delta)}\right)\right)
\]
is sufficiently large, by applying the high probability bound in Theorem~\ref{thrm:generalization_barw_iht} we obtain that with probability at least $1-\delta-\exp\{-c_0 n\}$,
\[
F(w^{(T)}_{S,k}) - F(\bar w) \le \mathcal{O}\left(\frac{\lambda_{\max}(\Sigma)}{\lambda^2_{\min}(\Sigma)}\left(\frac{k\sigma^2\log (p/\delta)}{n}\right)\right).
\]
This proves the desired bounds.
\end{proof}

\vspace{0.2in}

Next we prove Corollary~\ref{corol:generalization_barw_logisticreg_fast} as an application of Theorem~\ref{thrm:generalization_barw_iht} to sparse logistic regression models.

\begin{proof}[Proof of Corollary~\ref{corol:generalization_barw_logisticreg_fast}]
Let $\xi=\{x,y\}$ in which $x$ is zero-mean sub-Gaussian with covariance matrix $\Sigma \succ 0$ and $y\in\{-1,1\}$ is generated by $\mathbb{P}(y|x; \bar w) = \frac{\exp (2y\bar w^\top x)}{1+\exp (2y\bar w^\top x)}$. The logistic loss function at $\xi_i$ is given by $\ell(w;\xi_i) = \log(1+\exp(-2y_i w^\top x_i))$. We first show that $\nabla F(\bar w)=\mathbb{E}_{\xi}\left[\nabla \ell(\bar w;\xi)\right]=0$. Indeed,
\[
\begin{aligned}
&\mathbb{E}_{\xi}\left[\nabla \ell(\bar w;\xi)\right] \\
=& \mathbb{E}_{x,y}\left[ \nabla \log(1+\exp(-2y \bar w^\top x))\right] = \mathbb{E}_{x} \left[\mathbb{E}_{y\mid x}\left[\nabla \log(1+\exp(-2y \bar w^\top x))\mid x\right]\right]\\
=& \mathbb{E}_{x} \left[\mathbb{P}(y=1\mid x)\nabla \log(1+\exp(-2 \bar w^\top x)) + \mathbb{P}(y=-1\mid x)\nabla \log(1+\exp(2 \bar w^\top x)) \right]\\
=& \mathbb{E}_{x} \left[\frac{\exp (2\bar w^\top x)}{1+\exp (2\bar w^\top x)}\frac{-2x \exp(-2 \bar w^\top x)}{1+\exp(-2 \bar w^\top x)} + \frac{1}{1+\exp (2\bar w^\top x)}\frac{2x \exp(2 \bar w^\top x)}{1+\exp(2 \bar w^\top x)} \right]=0.
\end{aligned}
\]
Next we show that $\nabla_j \ell(\bar w;\xi)=\frac{-2 y [x]_j\exp(-2y \bar w^\top x)}{1+\exp(-2y \bar w^\top x)}$ is a zero-mean sub-Gaussian random variable. Clearly, $\mathbb{E}[\nabla_j \ell(\bar w;\xi)]=0$. Since $y\in\{-1,1\}$ and $[x]_j$ is $\frac{\sigma^2}{32}$-sub-Gaussian, we can show the following
\[
\mathbb{P}\left( |\nabla_j \ell(\bar w;\xi)| \ge t\right) = \mathbb{P}\left( \frac{2 |[x]_j|\exp(-2y \bar w^\top x)}{1+\exp(-2y \bar w^\top x)} \ge t\right) \le \mathbb{P}\left(|[x]_j| \ge \frac{t}{2}\right) \le  2\exp\left(-\frac{4t^2}{\sigma^2}\right).
\]
Then based on the result~\cite[Lemma 1.5]{rigollet201518}  we know that for any $\lambda>0$,
\[
\mathbb{E}_{\xi}\left[\exp(\lambda\nabla_j \ell(\bar w;\xi))\right]\le \exp\left(\frac{\lambda^2\sigma^2}{2}\right),
\]
which shows that $\nabla_j \ell(\bar w;\xi)$ is $\sigma^2$-sub-Gaussian. This verifies the validness of Assumption~\ref{assump:gradient_sub_gaussian}.

By invoking Lemma~\ref{lemma:strong_convexity_subgaussian} we obtain that if $n \ge \frac{4\sigma^2kc_1 \log (p)}{\lambda_{\min}(\Sigma)}$, then it holds with probability at least $1- \exp\{-c_0 n\}$ that $F_S(w)$ is $\mu_{4k}$-strongly convex with $\mu_{4k}\ge \frac{\lambda_{\min}(\Sigma)}{\exp(4R)}$. It is standard to verify that $F$ and $F_S$ are $\mathcal{O}(1)$-smooth almost surely. Therefore, provided that
\[
T\ge\mathcal{O}\left(\frac{\exp(R)}{\lambda_{\min}(\Sigma)}\log \left(\frac{n \lambda_{\min}(\Sigma)}{k\exp(R)\sigma^2\log(p/\delta)}\right)\right)
\]
is sufficiently large, by applying the bound in Theorem~\ref{thrm:generalization_barw_iht} we obtain that the following bound holds with probability at least $1-\delta-\exp\{-c_0 n\}$ :
\[
F(w^{(T)}_{S,k}) - F(\bar w) \le \mathcal{O}\left(\frac{\exp(R)}{\lambda^2_{\min}(\Sigma)}\left(\frac{k \sigma^2 \log (p/\delta)}{n}\right)\right).
\]
This concludes the proof.
\end{proof}

\end{document}